\documentclass[journal]{IEEEtran}
\usepackage{cite}
\usepackage{amsmath}
\usepackage{amsfonts}
\usepackage{amsthm}
\usepackage[table]{xcolor}
\usepackage{tabularx}
\usepackage{multirow}
\usepackage{caption}
\usepackage{soul}
\usepackage{booktabs}
\usepackage{graphicx} 
\usepackage[ruled,vlined]{algorithm2e} 
\ifCLASSINFOpdf
\else
\fi
\begin{document}
%

\title{Heterogeneous domain adaptation: An unsupervised approach}
%
%
%

\author{Feng Liu,~\IEEEmembership{Student Member,~IEEE,}
        Guangquan~Zhang,
        and~Jie~Lu,~\IEEEmembership{Fellow,~IEEE}
\thanks{Feng Liu, Guangquan Zhang and Jie Lu are with the Centre for Artificial Intelligence, Faulty of Engineering and Information Technology, University of Technology Sydney, Sydney,
NSW, 2007, Australia, e-mail:  \{Feng.Liu; Guangquan.Zhang; Jie.Lu\}@uts.edu.au.}
}

%
%

\markboth{IEEE TRANSACTIONS ON NEURAL NETWORKS AND LEARNING SYSTEMS, Early Access, 2020}%
{Liu \MakeLowercase{\textit{et al.}}: Heterogeneous domain adaptation: An unsupervised approach}
%



\maketitle

\begin{abstract}
Domain adaptation leverages the knowledge in one domain - the source domain - to improve learning efficiency in another domain - the target domain. {Existing heterogeneous domain adaptation research is relatively well-progressed, but only in situations where the target domain contains at least a few labeled instances. In contrast, heterogeneous domain adaptation with an unlabeled target domain has not been well-studied.} To contribute to the research in this emerging field, this paper presents: (1) an unsupervised knowledge transfer theorem that guarantees the correctness of transferring knowledge; and (2) {a principal angle-based metric to measure the distance between two pairs of domains: one pair comprises the original source and target domains and the other pair comprises two homogeneous representations of two domains.} The theorem and the metric have been implemented in an innovative transfer model, called a \emph{Grassmann-Linear monotonic maps-geodesic flow kernel} (GLG), that is specifically designed for \emph{heterogeneous unsupervised domain adaptation} (HeUDA). The linear monotonic maps meet the conditions of the theorem and are used to construct homogeneous representations of the heterogeneous domains. The metric shows the extent to which the homogeneous representations have preserved the information in the original source and target domains. By minimizing the proposed metric, the GLG model learns the homogeneous representations of heterogeneous domains and transfers knowledge through these learned representations via a geodesic flow kernel. To evaluate the model, five public datasets were reorganized into ten HeUDA tasks across three applications: cancer detection, credit assessment, and text classification. The experiments demonstrate that the proposed model delivers superior performance over the existing baselines.
\end{abstract}


%
\IEEEpeerreviewmaketitle

\section{Introduction}
\label{sec:intro}
%
%
%
%
\IEEEPARstart{I}{n} the field of \emph{artificial intelligence} (AI), and particularly in machine learning, storing the knowledge learned by solving one problem and applying it to a similar problem is very challenging. For example, the knowledge gained from recognizing cars could be used to help recognize trucks, value predictions for US real estate could help predict real estate values in Australia, or knowledge learned by classifying English documents could be used to help classify Spanish documents. As such, transfer learning models \cite{Pan2010,Lu2015,ShaoZL15,liu2019butterfly} have received tremendous attention by scholars in object recognition \cite{Gong2014,Luo_IJCAI_18, Yan_TNN_18,Yang_TNN_16}, AI planning \cite{Zhuo2014}, reinforcement learning \cite{Bianchi2015,Nguyen2017,Chalmers_TNN_18}, recommender systems \cite{Zhao2017,Pan2013}, and natural language processing \cite{Zhao2014}. Compared to traditional single-domain machine learning models, transfer learning models have clear advantages. (1) {The knowledge learned from one domain - the source domain - can help improve prediction accuracy in another domain - the target domain - particularly when the target domain has scant data} \cite{Ma2014}, and, 2) knowledge from a labeled domain can help predict labels for an unlabeled domain, which may avoid a costly human labeling process \cite{Gopalan2014}.

Of the proposed transfer learning models, domain adaptation models have demonstrated good success in various practical applications in recent years \cite{Ghifary2017,Courty2017}. Most domain adaptation models focus on \emph{homogeneous unsupervised domain adaptation} (HoUDA); that is, where the source and target domains have similar, same-dimensionality feature spaces and there are no labeled instances the target domain \cite{domain_adaptation_bounds}. 
Nevertheless, given the time and cost associated with human labeling, target domains are heterogeneous\footnote{{In the field of domain adaptation, ``heterogeneity" often represents that 1) dimensionality of source and target domains are different and 2) features of two domains are disjoint.}} and unlabeled, which means most existing HoUDA models do not perform well on the majority of target domains. Thus, \emph{heterogeneous unsupervised domain adaptation} (HeUDA) models are proposed to handle the situation where target domain is heterogeneous and unlabeled.

However, existing HeUDA models need parallel sets to bridge two heterogeneous domains, i.e., there are very similar instances in both heterogeneous domains, which is not realistic in the real world. For example, credit assessment data is confidential and private, and the information of each instance cannot be accessed. Thus, we cannot find similar instances between two credit-assessment domains. Namely, parallel sets (needed by existing HeUDA models) do not exist in this scenario. To the best of our knowledge, little theoretical discussion has taken place in regard to the absence of a parallel set in the HeUDA setting. This \emph{gap} limits the ability of HeUDA models to be used in more scenarios.

The aim of this paper is to fill this gap by establishing a theoretical foundation for HeUDA models that predict labels for a heterogeneous and unlabeled target domain without parallel sets. 
We are motivated by the observation that two heterogeneous domains may come from one domain. Namely, features of two heterogeneous domains could be outputs of heterogeneous projections of features of the one domain (see Figure~\ref{fig: main_flow}). In the following two paragraphs, we present two examples to describe this observation.

Sentences written in Latin can be translated into sentences written in French and Spanish. The French and Spanish sentences have different representations but share a similar meaning. If the Latin sentences are labeled as ``positive", then the French and Spanish sentences are probably labeled as ``positive''. {In this example, we can construct a Latin domain using Latin sentences and the task (labeling sentences as ``positive" or ``negative"). Then, French domain and Spanish domain come from one domain: Latin domain, where French and Spanish domains consist of French and Spanish sentences (translated from Latin sentences) and the task (labeling sentences as ``positive" or ``negative"). 

Taking another example in real-world scenarios: human sentiment, as an underlying domain (to analyze whether a person is happy), is difficult to record accurately. We can only obtain its projection or representation on real events, such as Amazon product reviews and Rotten Tomatoes movie reviews. The Amazon product reviews and the Rotten Tomatoes movie reviews are two heterogeneous domains but come from an underlying domain: human sentiment.}

Based on this observation, we propose two key factors, $V$ and $D$, to reveal the similarity between two heterogeneous domains:
\begin{itemize}
    \item the variation ($V$) between the conditional probability density functions of both domains;
    \item the distance ($D$) between the feature spaces of the two heterogeneous domains.
\end{itemize}
In general, small $V$ means that two domains have similar ground-truth labeling functions and small $D$ means that two feature spaces are close. 

In this paper, we construct \emph{homogeneous representations} to preserve the original similarity (evaluated by $V$ and $D$) between two heterogeneous domains, while allowing knowledge to be transferred. 
We denote $V_{He}$, $V_{Ho}$, $D_{He}$ and $D_{Ho}$ by values of $V$ and $D$ of the original \emph{heterogeneous} (He) domains and the \emph{homogeneous} (Ho) representations. The basic assumption of unsupervised domain adaptation models is that two domains have similar ground-truth labeling functions. Hence, the constructed homogeneous representations must make $V_{Ho}\le V_{He}$. Similarly, $D_{Ho}\le D_{He}$ is expected, indicating that the distance between two feature spaces of the homogeneous representations is small. We mainly focus on how to construct the homogeneous representations where $V_{Ho} = V_{He}$ and $D_{Ho} = D_{He}$ (the exact homogeneous representations of two heterogeneous domains).

To ensure the efficacy of the homogeneous representations, this paper presents: (1) an unsupervised knowledge transfer theorem that guarantees the correctness of transferring knowledge (to make $V_{Ho} = V_{He}$); and (2) a principal angle-based metric to measure the distance between two pairs of domains: one pair comprises the original source and target domains and the other pair comprises two homogeneous representations of two domains (to help make $D_{Ho} = D_{He}$). Based on the constructed exact homogeneous representations of two heterogeneous domains, HoUDA models can be applied to transfer knowledge across the representations. 

\begin{figure}[!tp] 
\centering\includegraphics[width=3.5in]{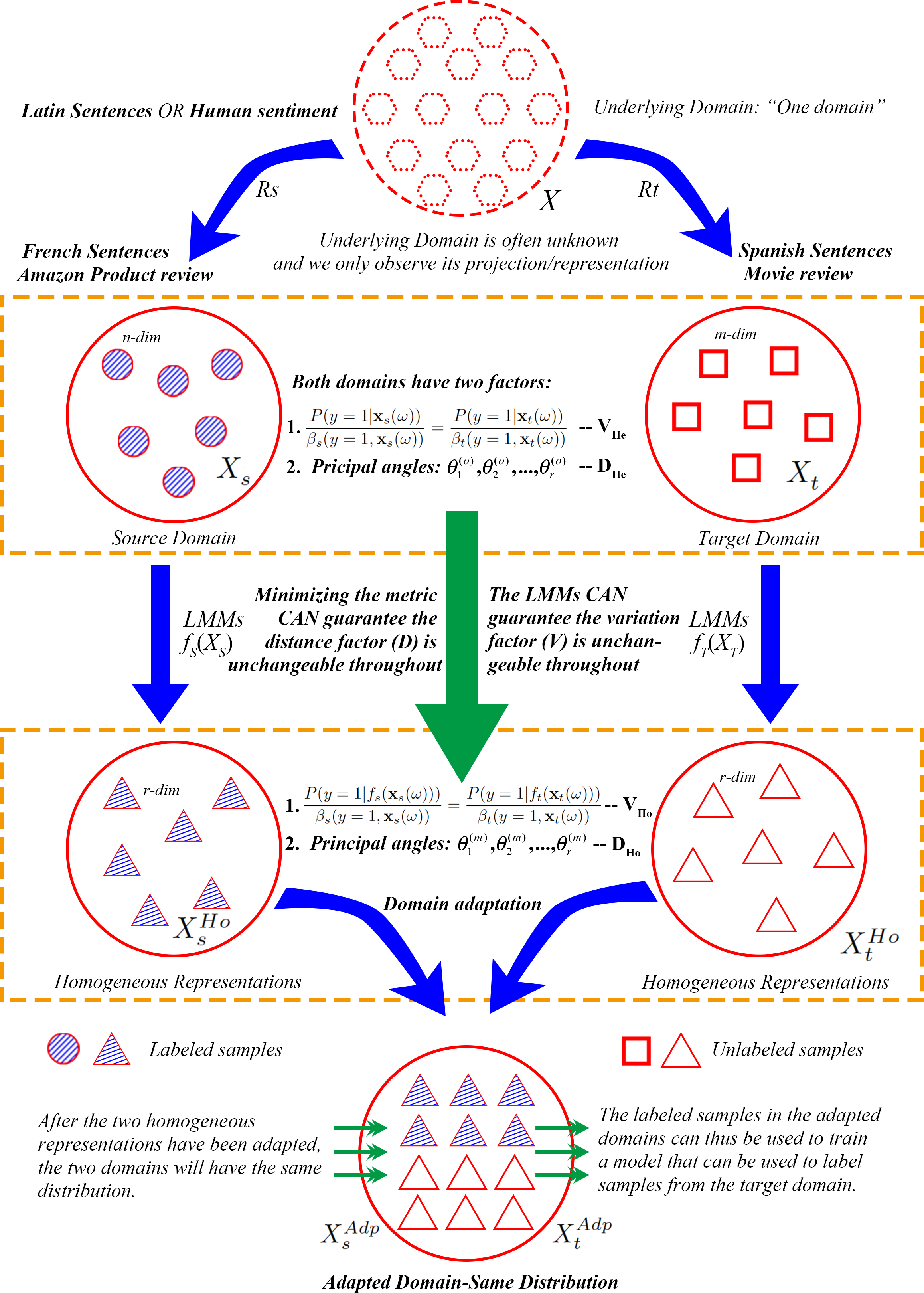} 
\caption{{The progress of the GLG model. The original source and target domains come from the same underlying domain (e.g., classifying Latin sentences or analyzing human sentiment). However, the underlying domain is hard to observe and we can only observe its projection/representation on two (or more) domains, e.g., two heterogeneous domains in this figure. Two factors are used to describe the similarity between two heterogeneous domains, {$V_{He}$} and {$D_{He}$}. Hence, we hope that two homogeneous representations will have the same similarity as the original domains. The LMMs can guarantee that the variation factor is unchangeable, and minimizing {$J_1$} (presented in Eq. {\eqref{eq: costI}}) can guarantee that the distance factor is unchangeable. After constructing homogeneous representations, GFK is applied to transfer knowledge across domains.}}\label{fig: main_flow} 
\end{figure} 

The unsupervised knowledge transfer theorem sets out the transfer conditions necessary to prevent negative transfer (to make $V_{Ho} = V_{He}$). \emph{Linear monotonic maps} (LMMs) meet the transfer conditions of the theorem and are therefore used to construct the homogeneous representations. Rather than directly measuring the distance between two heterogeneous feature spaces, the distance between two feature subspaces of different dimensions is measured using the principal angles of Grassmann manifold. This new distance metric reflects the extent to which the homogeneous representations have preserved the geometric relationship between the original heterogeneous domains (to make $D_{Ho} = D_{He}$). It is defined on two pairs of subspace sets; one pair of subspace sets reflects the original domains, the other reflects the homogeneous representations. 

Homogeneous representations of the heterogeneous domains are constructed by minimizing the distance metric based on the constraints associated with LMMs, i.e., minimize $\|D_{Ho} - D_{He}\|_{\ell_1}$ under the constraints $V_{Ho} = V_{He}$. Knowledge is transferred between the domains through the homogeneous representations via a \emph{geodesic flow kernel} (GFK) \cite{Gong2014}. The complete proposed HeUDA model incorporates all these elements and is called the Grassmann-LMM-GFK model - GLG for short. Figure \ref{fig: main_flow} illustrates the process of GLG.

To validate the efficacy of GLG, five public datasets were reorganized into ten tasks across three applications: cancer detection, credit assessment, and text classification. The experimental results reveal that the proposed model can reliably transfer knowledge across two heterogeneous domains when the target domain is unlabeled and there are no parallel sets. The main contributions of this paper are: 

1)	an effective heterogeneous unsupervised domain adaptation model, called GLG, that is able to transfer knowledge from a source domain to an unlabeled target domain in settings where both domains have heterogeneous feature spaces and are free of parallel sets;

2)	an unsupervised knowledge transfer theorem that prevents negative transfer for HeUDA models; and

3)	a new principal angle based metric shows the extent to which homogeneous representations have preserved the geometric distance between the original domains, and reveals the relationship between two heterogeneous (different-dimensionality) feature spaces.

This paper is organized as follows. Section II includes a review of the representative domain adaptation models. Section III introduces the GLG model, and its optimization process is presented in Section IV. Section V describes the experiments conducted to test the model's effectiveness. Section VI concludes the paper and discusses future works. Proofs of lemmas and theorems can be found in the Appendix.

\section{Related work}
In this section, homogeneous unsupervised domain adaptation models and heterogeneous domain adaptation models  which are most related to work are reviewed, and GLG is compared with these models.

\subsection{Homogeneous unsupervised domain adaptation}
To address HoUDA problem, there are four main techniques: the Grassmann-manifold method \cite{Gong2014,Gopalan2014,Fernando2013,Sun2015,Sun2016}, the integral-probability-metric method \cite{Gong2016,Long2016,Long2016a,Cao2018,Rozantsev2018}, the pseudo-labeling method \cite{KSaito_ICML17,Long_JDA,behbood2015multistep} and the adversarial-training method \cite{DANN_JMLR,Zhun_TIP_2019}. 
GFK, as a Grassmann-manifold-based model, seeks the best of all subspaces between the source and target domains, using the geodesic flow of a Grassmann manifold to find latent spaces through integration \cite{Gong2014}. 

\emph{Transfer component analysis} (TCA) \cite{Pan2011} applies \emph{maximum mean discrepancy} (MMD \cite{Gretton2012}, an integral probability metric) to measure the distance between the source and target feature spaces, and optimizes this distance to make sure the two domains are closer than before. 
\emph{Joint distribution adaptation} (JDA) \cite{Long_JDA} improves TCA by jointly matching marginal distributions and conditional distributions. 
\emph{Scatter component analysis} (SCA) \cite{Ghifary2017} extends TCA and JDA, and considers the between and within class scatter. 
\emph{Wasserstein Distance Guided Representation Learning} (WDGRL) \cite{shen2018wasserstein} minimizes the distribution discrepancy by employing Wasserstein Distance in neural networks. \emph{Deep adaptation networks} (DAN) \cite{Long_DAN_journal} and \emph{joint adaptation networks} (JAN) \cite{Long_JAN} employ MMD and deep neural networks to learn the best domain-invariant representations of two domains. 

{Asymmetric Tri-training domain adaptation} \cite{KSaito_ICML17}, as a pseudo-labeling-based model, is trained with labeled instances from a source domain and a pseudo-labeled target domain.

\emph{Domain-adversarial neural network} (DANN), as an adversarial-training-based model, is directly inspired by the theory on domain adaptation, suggesting that predictions must be made based on features that cannot discriminate between the training (source) and test (target) domains.

\subsection{Heterogeneous domain adaptation}
There are three types of heterogeneous domain adaptation models: \emph{heterogeneous supervised domain adaptation} (HeSDA), \emph{heterogeneous semi-supervised domain adaptation} (HeSSDA), and HeUDA. Following \cite{Li2014,Xiao2015}, {``heterogeneity" in the domain adaptation field often represents the source and target features as having different dimensionality and being disjoint. For example, if 1) German credit record has $24$ features and Australian credit record has $14$ features and 2) features from German credit record and Australian credit record are disjoint, then we say that that German credit record and Australian credit record are heterogeneous.}

HeSDA/HeSSDA aims to transfer knowledge from a source domain to a heterogeneous target domain, in which the two domains have different features. There is less literature on this setting than there is for homogeneous situations. The main models are \emph{heterogeneous spectral mapping} (HeMap) \cite{Shi2013}, manifold alignment-based models (MA) \cite{Wang2011}, \emph{asymmetric regularized cross-domain transformation} (ARC-t) \cite{Kulis2011}, \emph{heterogeneous feature augmentation} (HFA) \cite{Li2014}, co-regularized online transfer learning \cite{Zhao2014}, \emph{semi-supervised kernel matching for domain adaptation} (SSKMDA) \cite{Xiao2015}, the DASH-N model \cite{Nguyen2015}, Discriminative correlation subspace model \cite{Yan_IJCAI_17} and semi-supervised entropic Gromov-Wasserstein discrepancy \cite{Yan_IJCAI_18}.

Of these models, ARC-t, HFA and co-regularized online transfer learning only use labeled instances in both domains; the other models are able to use unlabeled instances to train models. 
HeMap works by using spectral embedding to unify different feature spaces across the target and source domains, even when the feature spaces are completely different \cite{Shi2013}. Manifold alignment derives its mapping by dividing the mapped instances into different categories according to the original observations \cite{Wang2011}. SSKMDA maps the target domain points to similar source domain points by matching the target kernel matrix to a submatrix of the source kernel matrix based on a Hilbert Schmidt Independence Criterion \cite{Xiao2015}. 

DASH-N is proposed to jointly learn a hierarchy of features combined with transformations that rectify any mismatches between the domains and has been successful in object recognition \cite{Xiao2015}. A discriminative correlation subspace model
is proposed to find the optimal discriminative correlation subspace for the source and target domain. \cite{Yan_IJCAI_18} presents a novel HeSSDA model by exploiting the theory of optimal transport, a powerful tool originally designed for aligning two different distributions. \emph{Progressive alignment} (PA) \cite{li2018_TNNLS_HDA} is implemented to learn representations of two heterogeneous domains with an unsupervised algorithm, but it still needs labeled instances from the target domain to train a final classifier which can handle possible negative transfer situations.

Unsupervised domain adaptation models based on homogeneous feature spaces have been widely researched. However, HeUDA models are rarely studied due to two shortcomings of current domain adaptation models: the feature spaces must be homogeneous, and there must be at least some labeled instances in the target domain (or there must be a parallel set in both domains). The hybrid heterogeneous transfer learning model \cite{Zhou2014} uses the information of the parallel set of both domains to transfer knowledge across domains. 

Domain
Specific Feature Transfer \cite{Wei2018} is designed to address the HeUDA problem when two domains have common features. \emph{Kernel canonical correlation analysis} (KCCA) \cite{Yeh2014} was proposed to address HeUDA problems when there are paired instances in the source and target domains, but KCCA is not valid when paired instances unavailable.
\emph{Shared fuzzy equivalence relations} (SFER) \cite{Liu_TFS} designs a novel fuzzy co-clustering method to separately cluster features of two domains into the same categories. Using these categories as a bridge, knowledge is transferred across two domains.

\vspace{-0.5em}\subsection{Comparison to related work}

The SCA model, as an example of existing HoUDA models, incorporates a fast representation learning algorithm for unsupervised domain adaptation. However, this model can only transfer knowledge across homogeneous domains.

The SSKMDA model, as an example of existing HeSSDA models, however, relies on labeled instances in the target domain to help  
correctly measure the similarity between two heterogeneous feature spaces (i.e., $V$ and $D$ in Section~\ref{sec:intro}). Compared to SSKMDA, GLG relies on the unsupervised knowledge transfer theorem to maintain $V$ and the principal angles of a Grassmann manifold to measure the distance ($D$) between two heterogeneous feature spaces. Therefore, GLG does not require any labeled instances in the target domain. 

Compared to existing HeUDA models, e.g. KCCA, it can transfer knowledge between two heterogeneous domains when both domains have paired instances and the target domain is unlabeled. However, the models are invalid {when there are no paired instances.} GLG is designed to transfer knowledge without needing paired instances and is based on a theorem that prevents negative transfer. 

\section{Heterogeneous Unsupervised domain adaptation}

Our HeUDA model, called GLG, is built around an unsupervised knowledge transfer theorem that avoids negative transfer through a variation factor $V$ that measures the difference between the conditional probability density functions in both domains. The unsupervised knowledge transfer theorem guarantees \emph{linear monotonic maps} (LMMs) against negative transfer once used to construct homogeneous representations of the heterogeneous domains (because $V_{Ho}=V_{He}$). A metric, which reflects the distance between the original domains and the homogeneous representations, ensures that the distance factor $D_{He}$ between the original domains is preserved (i.e., $D_{Ho} = D_{He}$). Thus, the central premise of the GLG model is to find the best LMM such that the distance between the original domains is preserved. 


\subsection{Problem setting and notations}

Following our motivation (two heterogeneous domains may come from one domain), we first give a distribution $\mathcal{P}$ over a multivariate random variable $\mathbf{X}$ defined on an instance set $\mathcal{X}$, $\mathbf{X}: \mathcal{X}\rightarrow\mathbb{R}^k$ and a labeling function $f: \mathbb{R}^k\rightarrow[0,1]$. The value of $f(\mathbf{X})$ corresponds to the probability that the label of $\mathbf{X}$ is 1. {In this paper, we use $\omega$ to represent a subset of $\mathcal{X}$, i.e. $\omega \subset \mathcal{X}$, and use $P(\mathbf{Y}=1|\mathbf{X})$ to represent $f(\mathbf{X})$, where $\mathbf{Y}$ is the label of $\mathbf{X}$ and the value of $\mathbf{Y}$ is $-1$ or $1$. The multivariate random variables corresponding to features of two heterogeneous domains are images of $\mathbf{X}$:
\begin{align}
\mathbf{X_s} = R_s(\mathbf{X}),~~\mathbf{X_t} = R_t(\mathbf{X}),
\end{align}
where $R_s: \mathbb{R}^k\rightarrow\mathbb{R}^m$, $R_t: \mathbb{R}^k\rightarrow\mathbb{R}^n$, $\mathbf{X_s}\sim\mathcal{P}_s$ and $\mathbf{X_t}\sim\mathcal{P}_t$. In the heterogeneous unsupervised domain adaptation setting, $m \neq n$ and we can observe a source domain $\mathbf{D_s}=\{(x_{si},y_{si})\}_{i=1}^{N}$ and a target domain $\mathbf{D_t}=\{(x_{ti})\}_{i=1}^{N}$, where $x_{si}\in\mathbb{R}^m$, $x_{ti}\in\mathbb{R}^n$ are observations of the multivariate random variables $\mathbf{X_s}$ and $\mathbf{X_t}$, respectively, and $y_{si}$, taking value from $\{-1,1\}$, is the label of $x_{si}$. $X_s=\{(x_{si})\}_{i=1}^{N}$ builds up a feature space of $\mathbf{D_s}$ and $X_t=\{(x_{ti})\}_{i=1}^{N}$ builds up a feature space of $\mathbf{D_t}$ and $Y_s=\{(y_{si})\}_{i=1}^{N}$ builds up of a label space of $\mathbf{D_s}$. In the following section, $\mathbf{D_s}=(X_s, Y_s)$ and $\mathbf{D_t}=(X_t)$ for short. The HeUDA problem is how to use $\mathbf{D_s}$ and $\mathbf{D_t}$ to label each $x_{ti}$ in $\mathbf{D_t}$.

{In the language example (see  Section~\ref{sec:intro}), $\mathcal{X}$ represents sentences written in Latin and $\omega$ is a subset to collect some Latin sentences from $\mathcal{X}$. $\mathbf{X}$ is a multivariate random variable and represents the Latin representations of sentences in $\mathcal{X}$. Since we consider that French and Spanish sentences are translated from Latin sentences, $\mathbf{X_s}$ is the French representations of sentences in $\mathcal{X}$ and $\mathbf{X_t}$ be the Spanish representations of sentences in $\mathcal{X}$. It should be noted that, in general, Latin sentences and French (or Spanish) sentences are disjoint. However, in this example, French (or Spanish) sentences are translated from Latin sentences, which means that French (or Spanish) sentences and Latin sentences are associated.}

\subsection{Unsupervised knowledge transfer theorem for HeUDA}
\label{sec:theorem}

This subsection first presents the relationships between $P(\mathbf{Y}=1|\mathbf{X})$ and $P(\mathbf{Y}=1|\mathbf{X_s})$ (or $P(\mathbf{Y}=1|\mathbf{X_t})$) and then gives the definition of the variation factor ($V$) between $P(\mathbf{Y}=1|\mathbf{X_s})$ and $P(\mathbf{Y}=1|\mathbf{X_t})$. Based on $V$, we propose the unsupervised knowledge transfer theorem for HeUDA. 

Given a measurable subset $\omega\subset\mathcal{X}$, we can obtain the probability $c(\omega) = P(\mathbf{Y}=1|\mathbf{X}(\omega))$. We expect that the probability $P(\mathbf{Y}=1|R_s(\mathbf{X}(\omega)))$ and $P(\mathbf{Y}=1|R_t(\mathbf{X}(\omega)))$ will be around $c(\omega)$. If $\omega$ is regarded as the Latin sentences mentioned in Section~\ref{sec:intro}, $\mathbf{X}_s = R_s(\mathbf{X}(\omega))$ and $\mathbf{X}_t = R_t(\mathbf{X}(\omega))$ are French and Spanish representations of the Latin sentences. If the Latin sentences are labeled as ``positive" ($\mathbf{Y}=1$), we of course expect that the French and Spanish sentences will have a high probability of being labeled as ``positive". To ensure this, $\forall \omega\subset\mathcal{X}$, we assume the following equality holds.
\begin{align}
\label{eq: equality_main}
\frac{P(\mathbf{Y}=1|\mathbf{X}_s(\omega))}{\beta_s(\mathbf{Y}=1,\mathbf{X}_s(\omega))} = \frac{P(\mathbf{Y}=1|\mathbf{X}_t(\omega))}{\beta_t(\mathbf{Y}=1,\mathbf{X}_t(\omega))} = c(\omega),
\end{align}
where $\beta_s(\mathbf{Y}=1,\mathbf{X}_s(\omega))$ and $\beta_t(\mathbf{Y}=1,\mathbf{X}_t(\omega))$ are two real-value functions. Since two heterogeneous domains have a similar task (i.e., labeling sentences as ``positive" or ``negative"), we know $\beta_s(\mathbf{Y}=1,\mathbf{X}_s(\omega))$ and $\beta_t(\mathbf{Y}=1,\mathbf{X}_t(\omega))$ should be around $1$ and have following properties for any $\omega$.
\begin{align}
\label{eq: negativeT}
&{\beta_s(\mathbf{Y}=1,\mathbf{X}_s(\omega))} \neq \frac{1-c(\omega)}{c(\omega)} \nonumber \\ 
&\textnormal{or}~~{\beta_t(\mathbf{Y}=1,\mathbf{X}_t(\omega))} \neq \frac{1-c(\omega)}{c(\omega)}.
\end{align}
{The properties described in {\eqref{eq: negativeT}} ensure that it is beneficial to transfer knowledge from the source domain to the target domain. If we do not have both properties described in {\eqref{eq: negativeT}}, i.e.,  ${\beta_s(\mathbf{Y}=1,\mathbf{X}_s(\omega))} = (1-c(\omega))/{c(\omega)}$, we will have $P(\mathbf{Y}=1|\mathbf{X}_s(\omega)) = 1-c(\omega) = P(\mathbf{Y}=-1|\mathbf{X}(\omega))$, indicating that positive Latin sentences are represented by negative French sentences.} 
Based on \eqref{eq: equality_main}, we define the variation factor $V_{He}(P(\mathbf{Y}=1|\mathbf{X}_s(\omega)), P(\mathbf{Y}=1|\mathbf{X}_t(\omega)))$ as follows.
\begin{align}
\label{eq: V_factor}
V&_{He}(P(\mathbf{Y}=1|\mathbf{X}_s(\omega)), P(\mathbf{Y}=1|\mathbf{X}_t(\omega))) \nonumber \\ 
&= \big|P(\mathbf{Y}=1|\mathbf{X}_s(\omega))- P(\mathbf{Y}=1|\mathbf{X}_t(\omega))\big| \nonumber \\ 
&= c(\omega)\big|\beta_s(\mathbf{Y}=1,\mathbf{X}_s(\omega))-\beta_t(\mathbf{Y}=1,\mathbf{X}_t(\omega))\big|.
\end{align}

To study how to correctly transfer knowledge across two heterogeneous domains, we first give a definition of extreme negative transfer to show the worst case.
\newtheorem{mydef}{Definition}

\begin{mydef}[{Extreme negative transfer}]\label{def: extreme_negative_transfer}
Given $\mathbf{X_s}\sim\mathcal{P}_s$,  $\mathbf{X_t}\sim\mathcal{P}_t$ and Eq.~\eqref{eq: equality_main}, if, $\forall\omega\subset \mathcal{X}$, $f_s(\mathbf{X_s}): \mathbb{R}^m\rightarrow\mathbb{R}^r$ and $f_t(\mathbf{X_t}): \mathbb{R}^n\rightarrow\mathbb{R}^r$ satisfy
\begin{align*}
P(\mathbf{Y}=1|f_s(\mathbf{X}_s(\omega))= P(\mathbf{Y}=-1|f_t(\mathbf{X}_t(\omega)),
\end{align*}
then we call that $f_s(\mathbf{X_s})$ and $f_t(\mathbf{X_t})$ cause extreme negative transfer.
\end{mydef}

{Based on Definition {\ref{def: extreme_negative_transfer}}, if extreme negative transfer happens, we will transfer incorrect knowledge across two domains. In experiments, we can use target-domain classification accuracy to quantify extreme negative transfer: lower accuracy means that extremer negative transfer happens. Section~\ref{sec:exp_1_RMG} shows the consequence caused by extreme negative transfer.}

However, we cannot quantify extreme negative transfer without presence of labeled data in target domain. Thus, to avoid the extreme negative transfer in advance, we present the heterogeneous unsupervised domain adaptation condition as follows. Satisfying this condition means that the knowledge will be transferred in expected way.

\begin{mydef}[HeUDA condition]\label{def: HeUDA_condition}
Given $\mathbf{X_s}\sim\mathcal{P}_s$,  $\mathbf{X_t}\sim\mathcal{P}_t$ and the Eq.~\eqref{eq: equality_main}, if there are two maps $f_s(\mathbf{X_s})$ and $f_t(\mathbf{X_t})$, then, $\forall\omega\subset \mathcal{X}$, the heterogeneous unsupervised domain adaptation condition can be expressed by the following equation.
\begin{align}\label{eq: HeUDA_condition}
\frac{P(\mathbf{Y}=1|f_s(\mathbf{X}_s(\omega)))}{\beta_s(\mathbf{Y}=1,\mathbf{X}_s(\omega))} = \frac{P(\mathbf{Y}=1|f_t(\mathbf{X}_t(\omega)))}{\beta_t(\mathbf{Y}=1,\mathbf{X}_t(\omega))} = c(\omega),
\end{align}
where $\omega$ is a measurable set.
\end{mydef}

If this condition is satisfied, it is clear that
\begin{align*}
&P(\mathbf{Y}=1|f_s(\mathbf{X}_s(\omega)) \neq P(\mathbf{Y}=-1|f_t(\mathbf{X}_t(\omega)),
\end{align*}
and
\begin{align*}
&V_{Ho}(P(\mathbf{Y}=1|f_s(\mathbf{X}_s(\omega)), P(\mathbf{Y}=1|f_t(\mathbf{X}_t(\omega))) \\
&= c(\omega)\big|\beta_s(\mathbf{Y}=1,\mathbf{X}_s(\omega))-\beta_t(\mathbf{Y}=1,\mathbf{X}_t(\omega))\big|,
\end{align*}
indicating that $f_s$ and $f_t$ will not cause extreme negative transfer and $V_{He}=V_{Ho}$.
\newtheorem{myrem}{Remark}
\begin{myrem}
{The HeUDA condition defined in Definition {\ref{def: HeUDA_condition}} is a sufficient condition to correctly transfer knowledge across two heterogeneous domains, but it is not a necessary condition. Although we could define more HeUDA conditions (sufficient conditions) to correctly transfer knowledge across two heterogeneous domains, we cannot find maps to satisfy every condition. In this paper, the HeUDA condition described in Definition {\ref{def: HeUDA_condition}} can be satisfied by the proposed mapping function: linear monotonic map (defined in Section~\ref{sec:proposed_model}), which means that we find a practical way to correctly transfer knowledge across two heterogeneous domains.}
\end{myrem}
Although Definition \ref{def: HeUDA_condition} provides the basic transfer condition in HeUDA scenario, we still need to determine which kinds of map (i.e., $f_s$ and $f_t$) satisfy this condition. To explore one such map, we propose monotonic maps as follows:
\begin{mydef}[monotonic map]\label{def: monotonic_map}
If a map $f: \mathbb{R}^m\rightarrow\mathbb{R}^r$ satisfies the following condition
\begin{align*}
x_i<x_j \Rightarrow f(x_i) < f(x_j),
\end{align*}
where $(x_i,<)$ and $(f(x_i),<)$ are binary relations and ``$<$" is a strict partial order over $\mathbb{R}^m$ and $f(\mathbb{R}^m)$, then the map $f$ is a monotonic map.
\end{mydef}
The proposed unsupervised knowledge transfer theorem follows, based on Definition \ref{def: monotonic_map}.

\newtheorem{mythm}{Theorem}
\begin{mythm}[unsupervised knowledge transfer theorem] \label{thm: HeUDA_theorem}
Given $\mathbf{X_s}\sim\mathcal{P}_s$,  $\mathbf{X_t}\sim\mathcal{P}_t$ and the Eq.~\eqref{eq: equality_main}, if there are two maps $f_s(\mathbf{X_s}): \mathbb{R}^m\rightarrow\mathbb{R}^r$ and $f_t(\mathbf{X_t}): \mathbb{R}^n\rightarrow\mathbb{R}^r$ satisfy that 

1) $f_s(\mathbf{X_s})$ and $f_t(\mathbf{X_t})$ are monotonic maps;

2) $f_s^{-1}(f_s(\mathbf{X_s}))=\mathbf{X_s}$ and $f_t^{-1}(f_t(\mathbf{X_t}))=\mathbf{X_t}$;\newline
then $f_s(\mathbf{X_s})$ and $f_t(\mathbf{X_t})$ satisfy the heterogeneous unsupervised domain adaptation conditions.
\end{mythm}

Based on Theorem \ref{thm: HeUDA_theorem}, we demonstrate a choice $f_s(\mathbf{X_s})$ and $f_t(\mathbf{X_t})$ to satisfy the heterogeneous unsupervised domain adaptation condition, and highlight the sufficient conditions for reliable unsupervised knowledge transfer. If a mapping function from heterogeneous domains to homogeneous representations satisfies two conditions in Theorem \ref{thm: HeUDA_theorem}, it can transfer knowledge across domains with theoretical reliability.

\subsection{Principal angle-based measurement between heterogeneous feature spaces}
\label{sec:distance}

In this subsection, the method for measuring the distance between two subspaces is introduced. On a Grassmann manifold $G_{N,m}$ (or $G_{N,n}$), subspaces with $m$ (or $n$) dimensions of $\mathbb{R}^N$ are regarded as points in $G_{N,m}$ (or $G_{N,n}$). This means that measuring the distance between two subspaces can be calculated by the distance between those two points on the Grassmann manifold. First, the subspaces spanned by $X_s$ and $X_t$ are confirmed using singular value decomposition (SVD). The distance between the spanned subspaces $A = span(X_s)$ and $B = span(X_t)$ can then be calculated in terms of the corresponding points on the Grassmann manifold. 

There are two HoUDA models that use a Grassmann manifold in this way: DAGM and GFK. DAGM was proposed by Gopalan et al. \cite{Gopalan2014}. GFK was proposed by Gong and Grauman \cite{Gong2014}. Both have one shortcoming: the source domain and the target domain must have feature spaces of the same dimension, mainly due to the lack of geodesic flow on $G_{N,m}$ and $G_{N,n}$ ($m\neq n$). In \cite{Ye2016}, Ye and Lim successfully proposed the principal angles between two different dimensional subspaces, which helps measure the distance between two heterogeneous feature spaces consisting of $X_s$ and $X_t$. Principal angles for heterogeneous subspaces are defined as follows.
\begin{mydef}[principal angles for heterogeneous subspaces \cite{Ye2016}]\label{def: pri_angles}
Given two subspaces $A\in G_{N,m}$ and $B\in G_{N,n}$ ($m\neq n$), which form the matrixes $A\in \mathbb{R}^{N\times m}$ and $B\in \mathbb{R}^{N\times n}$, the $i^{th}$ principal vectors $(p_i, q_i)$, $i = 1, …, r,$ are defined as solutions for the optimization problem $(r = \min(n, m))$:
\begin{align}
\label{eq:opt_principal_A}
&\max~~p^Tq \nonumber \\
&s.~t.~~p\in A,~p^Tp_1=...=p^Tp_{i-1},~\|p\| = 1, \\
&~~~~~~~q\in B,~q^Tq_1=...=q^Tq_{i-1},~\|q\| = 1 \nonumber,
\end{align}
Then, the principal angles for heterogeneous subspaces are defined as 
\begin{align*}
cos\theta_i = p_i^Tq_i,~i=1,...,r.
\end{align*}
\end{mydef}
Ye and Lim \cite{Ye2016} proved that the optimization solution to \eqref{eq:opt_principal_A} can be computed using SVD. Thus, we can calculate the principal angles between two different-dimensionality subspaces, and this idea forms the distance factor $D$ mentioned in Section~\ref{sec:intro}. To perfectly define distances between subspaces of different dimensions, Ye and Lim used two Schubert varieties to prove that all the defined distances in subspaces of the same dimensions are also correct when the dimensionalities differ. This means we can calculate a distance between two subspaces of different dimensions using the principal angles defined in Definition \ref{def: pri_angles}. Given $A=span(X_s)$ and $B = span(X_t)$, the distance vector between $X_s$ and $X_t$ is defined as a vector containing principal angles between $A$ and $B$, which has the following expression.
\begin{align*}
D_{He}(X_s, X_t) = arccos([\sigma_1(A^TB), \sigma_2(A^TB), ..., \sigma_r(A^TB)]),
\end{align*}
where $r = \min(n, m)$, $\sigma_i(A^TB)$ is the $i^{th}$ singular value of $A^TB$ computed by SVD (the $i^{th}$ principal angles $\theta_i=arccos(\sigma_i(A^TB))$). 

If we can find two maps $f_s$ and $f_t$ that satisfy the conditions of Theorem \ref{thm: HeUDA_theorem}, we can obtain the $D_{Ho}$ as follows.
\begin{align*}
&D_{Ho}(f_s(X_s), f_t(X_t)) \\
= &~arccos([\sigma_1(C^TD), \sigma_2(C^TD), ..., \sigma_r(C^TD)]),
\end{align*}
where $C=span(f_s(X_s))$ and $D = span(f_t(X_t))$. Hence, we can measure the distance between $D_{He}$ and $D_{Ho}$ via these singular values of matrix $A^TB$ and $C^TD$.

\begin{myrem}
{The distance $D_{He}(X_s, X_t)$ defined in this subsection aims to describe a geometric relationship between $X_s$ and $X_t$. Compared to KL divergence, which estimates the distance between probability distributions, $D_{He}(X_s, X_t)$ has the following differences.}

{a) $D_{He}(X_s, X_t)$ is a vector that contains principal angles between a subspace spanned by $X_s$ and a subspace spanned by $X_t$, which means that it describes a geometric relationship between $X_s$ and $X_t$. However, KL divergence is a real number to describe a relationship between $X_s$ and $X_t$ from a probability perspective, so, $D_{He}(X_s, X_t)$ and KL divergence have different aims.

b) $D_{He}(X_s, X_t)$ is able to describe a geometric relationship between $X_s$ and $X_t$ when $X_s$ and $X_t$ have different dimensionalities (e.g., the dimensionality of $X_s$ is 24 and the dimensionality of $X_t$ is 14). However, KL divergence  can only be computed when $X_s$ and $X_t$ have the same dimensionalities (e.g., the dimensionality of $X_s$ is 14 and the dimensionality of $X_t$ is 14). This is why it is necessary to define a new distance to describe the relationship between two heterogeneous feature spaces from two heterogeneous domains. To the best of our knowledge, there is little discussion about the relationship between two different-dimensionality distributions.}
\end{myrem}
\subsection{The proposed HeUDA model}
\label{sec:proposed_model}
With the unsupervised knowledge transfer theorem that ensures the reliability of heterogeneous unsupervised domain adaptation, and with the principal angles of Grassmann manifolds explained, we now turn to the proposed model, GLG. 
The optimization solution for GLG is outlined in Section~\ref{sec:opt_GLG}. 

A common idea for finding the homogeneous representations of heterogeneous domains is to find maps that can project feature spaces of different dimensions (heterogeneous domains) onto feature spaces with same dimensions. {However, most heterogeneous domain adaptation models require at least some labeled instances or paired instances in the target domain to maintain the relationship between the source and target domains.} Thus, the key to a HeUDA model is to find a few properties that can be maintained between the original domains and the homogeneous representations. 

Here, these two factors are the variation factor ($V_{He}$ and $V_{Ho}$ defined in Section~\ref{sec:theorem}) and the distance factor ($D_{He}$ and $D_{Ho}$ defined in Section~\ref{sec:distance}). Theorem \ref{thm: HeUDA_theorem} determines the properties the maps should satisfy to make $V_{He}=V_{Ho}$ and principal angles shows the distance between two heterogeneous (or homogeneous) feature spaces ($D_{He}$ and $D_{Ho}$). However, there are still two concerns: 1) which type of mapping function is suitable for Theorem \ref{thm: HeUDA_theorem}; and 2) which properties should the map maintain between the original domains and the homogeneous representations. The first concern with the unsupervised knowledge transfer theorem is addressed by selecting LMMs as the map of choice.

\newtheorem{mylem}{Lemma}
\begin{mylem}[linear monotonic map]\label{lem: LMMs}
Given a map $f: \mathbb{R}^m\rightarrow\mathbb{R}^r$ with form $f(x) = xU^T$, $f(x)$ is a monotonic map if and only if $U>0$ or $U<0$, where $x\in\mathbb{R}^m$ 
and $U\in \mathbb{R}^{r\times m}$.
\end{mylem}
Since the defined map in Lemma \ref{lem: LMMs} only uses $U$ and according to the generalized inverse of a matrix, the matrix $f(X_s)$ satisfies $f^{-1}(f(X_s))=X_s$. Therefore, we can prove that LMMs satisfy the conditions in Theorem \ref{thm: LMM_theorem}.
\begin{mythm}[LMM for HeUDA]\label{thm: LMM_theorem}
Given $\mathbf{X_s}\sim\mathcal{P}_s$,  $\mathbf{X_t}\sim\mathcal{P}_t$ and Eq.~\eqref{eq: equality_main}, if there are two maps $f_s(\mathbf{X_s}): \mathbb{R}^m\rightarrow\mathbb{R}^r$ and $f_t(\mathbf{X_t}): \mathbb{R}^n\rightarrow\mathbb{R}^r$ are LMMs, then $f_s(\mathbf{X_s})$ and $f_t(\mathbf{X_t})$ satisfy the HeUDA condition.
\end{mythm}

\begin{myrem}
From this theorem and the nature of LMMs, we know this positive map can better handle datasets that have many monotonic samples because the probabilities in these monotonic samples can be preserved without any loss. The existence of these samples offers the greatest probability of preventing negative transfers.
\end{myrem}

 Theorem \ref{thm: LMM_theorem} addresses the first concern and provides a suitable map, such as the map in Lemma \ref{lem: LMMs}, to project two heterogeneous feature spaces onto the same dimensional feature space. It is worthwhile showing that an LMM is just one among many suitable maps for Theorem 1. A nonlinear map can also be used to construct the map, as long as the map is monotonic. In future work, we intend to explore additional maps suitable for other HeUDA models. 

This brings us to the second concern: which properties can be maintained during the mapping process between the original domains and the homogeneous representations? As mentioned above, the principal angles play a significant role in defining the distance between two subspaces on a Grassmann manifold, and in explaining the projection between them \cite{Wong1967}. Ensuring the principal angles remain unchanged is thus one option for maintaining some useful properties. 

Specifically, for any two pairs of subspaces ($A, B$) and ($C, D$), if the principal angles of ($A, B$) and ($C, D$) are the same (implying that min\{dim($A$), dim($B$)\} = min\{dim($C$), dim($D$)\}, dim($A$) represents the dimension of $A$), then the relationship between $A$ and $B$ can be regarded as similar to the relationship between $C$ and $D$. Based on this idea, the definition of measurement $\mathcal{D}$, which describes the relationships between two pairs of subspaces, follows.

\begin{mydef}[measurement between subspace pairs]\label{def: Measure_pairs}
Given two pairs of subspaces ($A,B$) and ($C,D$), the measurement $\mathcal{D}$(($A,B$), ($C,D$)) between ($A, B$) and ($C, D$) is defined as
\begin{align}
\mathcal{D}((A,B), (C,D))=\sum_{i=1}^r\Big|\sigma_i (A^TB)- \sigma_i (C^TD)\Big|,
\end{align}
where $A, B, C$ and $D$ are subspaces in $\mathbb{R}^N$, $r$=min\{dim($A$), dim($B$), dim($C$), dim($D$)\} and $\sigma_i (A^TB)$ is the $i^{th}$ singular value of matrix $A^TB$ and represents the cosine value of the $i^{th}$ principal angle between $A$ and $B$.
\end{mydef}
\begin{myrem}
{The measurement $\mathcal{D}$ is defined on two pairs of two subspaces (e.g., pair 1: $(A, B)$ and pair 2: $(C, D)$, where $A, B, C$ and $D$ are subspaces) rather than two distributions. This distance describes the distance between two pairs of subspaces (e.g., relationships between $(A, B)$ and $(C, D)$), which is different with distance between probability distributions, such as KL divergence.}
\end{myrem}
Measurement $\mathcal{D}$ defined on $G_{N,*}^T \times G_{N,*}$ is actually a metric, as proven in the following theorem.

\begin{mythm}
\label{thm: metric}
($\mathcal{D}, G_{N,*}^T \times G_{N,*}$) is a metric space, where $G_{N,*}=\{A|A\in G_{N,i},~i=1,...,N-1\}$.
\end{mythm}


{In contrast to the metric proposed in {\cite{Ye2016}}, our metric focuses on the distance between two pairs of subspaces, such as $(A,B)$ and $(C,D)$, rather than two subspaces, such as $A$ and $B$. The proposed metric, especially designed for the HeUDA problem, shows the extent to which homogeneous representations have preserved the geometric distance between two heterogeneous feature spaces. However, the metric proposed in {\cite{Ye2016}} only focuses on the distance between two subspaces, such as $A$ and $B$.} The definition of the consistency of the geometric relationship with respect to the feature spaces of two domains can be given in terms of the metric $\mathcal{D}$ as follows.
\begin{mydef}[consistency of the geometric relationship]\label{def: consistency}
Given the source domain $\mathbf{D_s}=(X_s, Y_s)$ and the heterogeneous and unlabeled target domain $\mathbf{D_s}=(X_t)$, let $f_s(X_s)=X_sU_s^T$ and $f_t(X_t)=X_tU_t^T$, if $\forall \delta \in (0,\delta_0]$, $\exists\epsilon<\mathcal{O}(\delta_0)$ such that 
\begin{align} 
\label{eq: consistencyeq}
&\int_0^{\delta_0}\mathcal{D}\Big(  (S_{X_s^{\delta}},S_{X_t^{\delta}} ), (S_m(f_s, X_s^{\delta}),S_m(f_t,X_t^{\delta}) ) \Big)d\delta <\epsilon,
\end{align}
then we can say that $(X_s, X_t)$ and $(f_s(X_s),f_t(X_t))$ have consistent geometric relationship, where $S_{X^{\delta}}=span(X+\delta\cdot\textbf{1}_X)$, $S_m(f, X^{\delta})=span(f(X+\delta\cdot\textbf{1}_X))$, $U_s\in\mathbb{R}^{r\times m}$, $U_t\in\mathbb{R}^{r\times n}$, $r=min\{m,n\}$ and $\textbf{1}_X$ is a matrix of ones of the same size as $X$.
\end{mydef}

This definition precisely demonstrates how $f_s$ and $f_t$ influence the geometric relationship between the original feature spaces and the feature spaces of homogeneous representations. If there are slight changes in the original feature spaces, we hope feature spaces of the homogeneous representations will also see slight changes. If they do, it means that the feature spaces of the homogeneous representations are consistent with the geometric relationships of the two original feature spaces. Based on definitions of $D_{He}$ and $D_{Ho}$, \eqref{eq: consistencyeq} is expressed by
\begin{align}
&\int_0^{\delta_0}\mathcal{D}\Big(  (S_{X_s^{\delta}},S_{X_t^{\delta}} ), (S_m(f_s, X_s^{\delta}),S_m(f_t,X_t^{\delta} ) \Big)d\delta <\epsilon \nonumber \\
&\Leftrightarrow\int_0^{\delta_0}\big\|D_{He}(X_s^{\delta},X_t^{\delta})-D_{Ho}(f_s(X_s^{\delta}),f_t(X_t^{\delta})) \big\|_{\ell_1}d\delta<\epsilon.
\end{align}
To ensure the consistency of the geometric relationship of the two original feature spaces, we minimize the following cost function to ensure that we are able to find an $\epsilon$ that is less than $\mathcal{O}(\delta_0)$, such that $\int_0^{\delta_0}\mathcal{D}\big(  (S_{X_s^{\delta}},S_{X_t^{\delta}} ), (S_m(f_s, X_s^{\delta}),S_m(f_t,X_t^{\delta} ) \big)d\delta <\epsilon$ when there are slight changes $\delta\in(0,\delta_0]$ in the original feature spaces.

\begin{mydef}[cost function I]\label{def: costI}
Given the source domain $\mathbf{D_s}=(X_s, Y_s)$ and the heterogeneous and unlabeled target domain $\mathbf{D_s}=(X_t)$, let $f_s(X_s)=X_sU_s^T$ and $f_t(X_t)=X_tU_t^T$, the cost function $J_1$ of GLG is defined as 
\begin{align}\label{eq: costI}
&J_1(X_s,X_t;U_s,U_t)  \nonumber \\
&=\int_0^{\delta_0}\big\|D_{He}(X_s^{\delta},X_t^{\delta})-D_{Ho}(X_s^{\delta},X_t^{\delta}) \big\|_{\ell_1}d\delta \nonumber \\
&+ \frac12\lambda_s Tr(U_sU_s^T) + \frac12\lambda_t Tr(U_tU_t^T),
\end{align}
where ${X^{\delta}}=X+\delta\cdot\textbf{1}_X$, $U_s\in\mathbb{R}^{r\times m}$, $U_t\in\mathbb{R}^{r\times n}$, $r=min\{m,n\}$ and $\textbf{1}_X$ is a matrix of ones of the same size as $X$.
\end{mydef}
This definition shows the discrepancy between the original feature spaces and the feature spaces of the homogeneous representations via principal angles. If we use $\theta_i^{(o)}$ to represent the $i^{th}$ principal angle of the original feature spaces and  $\theta_i^{(m)}$ to represent the $i^{th}$ principal angle of the feature spaces of the homogeneous representations, $J_1$ measures $|cos(\theta_i^{(o)}) - cos(\theta_i^{(m)})|$ when the original feature spaces have slight changes. $Trace(U_s U_s^T)$ and $trace(U_t U_t^T )$ are used to smooth $f_s$ and $f_t$. $\lambda_s$ is set to $0.01/mr$, and $\lambda_t$ is set to $0.01/nr$. When $m = n$, $\lambda_s$ and $\lambda_t$ are set to 0.
From Definition \ref{def: costI}, it is clear that the maps $f_s(X_s)$ and $f_t(X_t)$ will ensure that all principal angles will change slightly as $J_1$ approaches 0, even when there is some disturbance of up to $\delta_0$. Thus, based on Theorem \ref{thm: LMM_theorem} and Definition \ref{def: costI}, the GLG model is presented as follows.
\newline
\newline
\textbf{Model (GLG)}. 
The model GLG aims to find $U_s\in \mathbb{R}^{r\times m}$, $U_t\in \mathbb{R}^{r\times n}$ to minimize the cost function $J_1$, as defined in \eqref{eq: costI}, while $f_s(X_s) = X_sU_s^T$ and $f_t(X_t) = X_tU_t^T$ are LMMs. GLG is expressed as
\begin{align*}
&\underset{U_s, U_t}{\min}~~J_1(X_s,X_t;U_s,U_t) \\
&s.~t.~~U_s>0~~\textnormal{and}~~U_t>0.
\end{align*}
$f_s(X_s)$ and $f_t(X_t)$ are the new instances corresponding to $X_s$ and $X_t$ in the homogeneous representations with a dimension of $r$. Knowledge is then transferred between $f_s(X_s)$ and $f_t(X_t)$ using GFK.

Admittedly, LMMs are somewhat restrictive maps because all elements in the $U$ must be positive numbers. However, we use LMMs to prevent negative transfers, which can significantly prevent very low prediction accuracy in the target domain. From the perspective of the entire transfer process, an LMM, as a positive map, is the only map that can help construct the homogeneous representations ($V_{He}=V_{Ho}$ and $D_{He}=D_{Ho}$). The GFK model provides the second map, which does not have such rigid restrictions and brings two homogeneous representations closer. Hence, the composite map (LMM+GFK) does not carry rigid restrictions and can therefore handle more complex problems. LMMs ensure correctness, thus avoiding negative transfer, and the GFK model improves the ability to transfer knowledge.
The following theorem demonstrates the relationship between GFK and GLG.
\begin{mythm}[degeneracy of GLG]
\label{thm: d_GLG}
Given the source domain $\mathbf{D_s}=(X_s, Y_s)$ and the heterogeneous and unlabeled target domain $\mathbf{D_s}=(X_t)$, if two domains are homogeneous ($m=n$), then the GLG model degenerates into the GFK model.
\end{mythm}


Since this optimization issue is related to subspaces spanned by the original instances ($X_s$ and $X_t$) and the subspaces spanned by the distributed instances ($X_s^\delta$ and $X_t^\delta$), determining the best way to efficiently arrive at an optimized solution is a difficult and complex problem. Section IV proposes the optimization algorithm, focusing on the solution to GLG.

\subsection{Discussion of definitions and theorems}
{Since GLG is built around several definitions and theorems, this subsection explains why one definition leads to another and how one theorem leads to other, as well as discussing the importance of these theoretical demonstrations.}

{Definition {\ref{def: HeUDA_condition}} gives the heterogeneous unsupervised domain adaptation condition (HeUDA condition). If this condition can be satisfied, the knowledge from a source domain will be correctly transferred to a heterogeneous target domain. Theorem {\ref{thm: HeUDA_theorem}} shows the kind of map that can satisfy the HeUDA condition given in Definition {\ref{def: HeUDA_condition}}. In Theorem {\ref{thm: HeUDA_theorem}}, a new map - monotonic map defined in Definition {\ref{def: monotonic_map}} - is involved to prove Theorem {\ref{thm: HeUDA_theorem}}. To find maps such as those presented in Theorem {\ref{thm: HeUDA_theorem}}, an LMM is proposed in Lemma {\ref{lem: LMMs}}, and Theorem {\ref{thm: LMM_theorem}} proves that LMMs can map two heterogeneous feature spaces to two homogeneous representations with theoretical guarantee. This leads to our first theoretical contribution: how to theoretically prevent negative transfer in the heterogeneous unsupervised domain adaptation setting.}

{To find the best LMMs for two heterogeneous feature spaces, principal angles, explained in Definition {\ref{def: pri_angles}}, are used to describe the distance between two heterogeneous feature spaces. A new measurement $\mathcal{D}$ is proposed in Definition {\ref{def: Measure_pairs}} to describe the relationships between the original heterogeneous feature spaces and the homogeneous representations that are mapped from the original heterogeneous feature spaces by LMMs. To maintain the principal angles between two original feature spaces, the cost function $J_1$ is proposed in Definition {\ref{def: costI}}. Theorem {\ref{thm: metric}} proves that $\mathcal{D}$ is also a metric, which ensures that minimizing $J_1$ is meaningful for maintaining the principal angles between two original feature spaces. Theorem {\ref{thm: metric}} also indicates that $J_1=0$ if source and target domains are homogeneous domains, which leads to Theorem {\ref{thm: d_GLG}}. Theorem {\ref{thm: metric}} and Theorem {\ref{thm: d_GLG}} lead to our second theoretical contribution: how to describe and maintain the geometric distance between two heterogeneous feature spaces.}

\subsection{Limitation of GLG}
{Practically, GLG can be extended to address multi-class classification problem since the procedure for constructing homogeneous representations of two heterogeneous domains does not involve $y_{si}$ (labels in a source domain).}

{However, using GLG to directly address multi-class classification problems does not provide sufficient theoretical guarantees. LMMs, key mapping functions in GLG, can only guarantee that the probability of label "+1" (denoted by $P1$) of an instance set, such as a subset $x_t$ belonging to $X_t$, will not change to $1-P1$ after mapping this instance set to its homogeneous representation ($f_t(x_t)$). For example, if $P1(x_t) = 0.6$, then $P1(f_t(x_t))$ only lies in the interval $(0.4,0.6]$, but, in the multi-class situation (considering $10$ classes), if $P1(x_t) = 0.1$, then $P1(f_t(x_t))$ will lie in the interval $[0.1,0.9)$. The interval $[0.1,0.9)$ is not accepted because it is too long. If GLG is directly used to address the multi-class classification problem, the accuracy in the target domain will be low. To address this problem, a new mapping function (e.g., $f_t(x_t)$) is needed to ensure that $P1(f_t(x_t))$ is close to $P1(x_t)$, which is difficult to satisfy in unsupervised and heterogeneous situation.} 

{In our future work, we aim to extend GLG to address multi-class classification problems by using label-noise learning models because an unlabeled target domain with predicted labels can be regarded as a domain with noisy labels.} 

\section{Optimization of GLG}
\label{sec:opt_GLG}
According to \eqref{eq: costI}, we need to calculate 1) $\partial\sigma_i(C^TD)/\partial U_s$, $\partial\sigma_i(C^TD)/\partial U_t$ and 2) the integration with respect to $\delta$ to minimize $J_1$ via a gradient descent algorithm, where $C=span(f_s(X_s^\delta))$, $D=span(f_t(X_t^\delta))$, $\delta\in(0,\delta_0]$ and $i=1, ..., r$.
Calculating $\partial\sigma_i(C^TD)/\partial U_s$ and $\partial\sigma_i(C^TD)/\partial U_t$ contains the process of spanning a feature space to become a subspace. Thus, when there are disturbances in an original feature space, the microscopic changes of the eigenvectors in an \emph{Eigen dynamic system} (EDS) need to be analyzed (Eigenvectors are used to construct the subspaces spanned by a feature space, i.e., $C$ and $D$). The following subsection discusses the microscopic analysis of an EDS.
\subsection{Microscopic analysis of an Eigen dynamic system}
In this section, we explore the extent of the changes in subspace $A = span(X)$ when the feature space ($X$) has suffered a disturbance, expressed as $\partial A /\partial X$. Without loss of generality, assume $A \in G_{N,n}$  (formed as an $\mathbb{R}^{N\times n}$ matrix) and $X \in \mathbb{R}^{N\times n}$, where $n$ is the number of features of $X$ and $N$ is the dimension of the whole space. In keeping with SVD, $A$ is the first $n$ columns of the eigenvectors of $XX^T$, which means we have the following equations: 
\begin{align*}
XX^Ty_i = y_i\lambda_i&,~i=1,...,n \\
y_i^Ty_i&=1,
\end{align*}
where $y_i$ is the $i^{th}$ column of $A$, and $\lambda_i$ is the eigenvalue corresponding to $y_i$.

It is clear that if $X$ is disturbed, due to equality, $y_i$ and $\lambda_i$ will change correspondingly. This equation represents a basic EDS, which is widely used in many fields. To microscopically analyze this equation, we differentiate it into
\begin{align}\label{eq: EDS}
\frac{\partial XX^T}{\partial X}y_i+XX^T \frac{\partial y_i}{\partial X} = y_i \frac{\partial \lambda_i}{\partial X} + \frac{\partial y_i}{\partial X} \lambda_i.
\end{align}
After a series of calculations, Lemma 2 is derived as follows.

\begin{mylem}[first-order derivatives of EDS]\label{lem: 1st_EDS}
Given a feature space $X \in \mathbb{R}^{N \times n}$, let $A = span(X) \in G_{N,n}$  (formed as an $\mathbb{R}^{N \times n}$ matrix), let $y_i$ be the $i^{th}$ column of $A$, and let $\lambda_i$ be the eigenvalue corresponding to $y_i$. The first-order derivatives of the EDS are 
\begin{align*}
\frac{\partial y_i}{\partial X}= &-(XX^T-\lambda_i I)^+  \frac{\partial XX^T}{\partial X} y_i, \\
&\frac{\partial \lambda_i}{\partial X}= y_i^T  \frac{\partial XX^T}{\partial X} y_i, \\
\end{align*}
where $(XX^T-\lambda_i I)^+$ is  the Moore-Penrose pseudoinverse of $XX^T-\lambda_i I$.
\end{mylem}

Based on Lemma \ref{lem: 1st_EDS}, we know the extent of the changes in subspace $A = span(X)$ when the feature space ($X$) has suffered a disturbance, expressed as $\partial A /\partial X$.

\subsection{Gradients of $J_1$}
With the proposed lemma, we obtain the derivative of cost function $J_1$ using following chain rules. For simplicity, $S_m^s$ is short for $S_m(f_s,X_s^\delta)$ and $S_m^t$ is short for $S_m(f_t,X_t^\delta)$.
\begin{align} \label{eq: main_gra}
&\frac{\partial J_1}{\partial (U_s)_{cd}}=
\int_{\delta=0}^{\delta_0}\frac{\partial J_1}{\partial \mathcal{D}} 
\sum_{i=1}^r\frac{\partial \mathcal{D}}{(\partial \sigma_i ((S_m^s)^T S_m^t) } \cdot \nonumber \\ 
&Tr\left(\Big(\frac{\partial \sigma_i ((S_m^s )^T S_m^t )}{\partial S_m^s}\Big)^T  
\frac{\partial S_m^s}{\partial (U_s)_{cd} }\right) d\delta +\lambda_s (U_s )_{cd}. 
\end{align}

The first and second terms of the right side can be easily calculated according to the definition of the cost function $J_1$. Using chain rules, the third term can be calculated by the following equations:

\begin{align}\label{eq: part 1}
\frac{\partial \sigma_i ((S_m^s )^T S_m^t )}{\partial (S_m^S)_{kl}}=
Tr\left(\Big(\frac{\partial \sigma_i ((S_m^s )^T S_m^t )}{\partial (S_m^s)^TS_m^s}\Big)^T  
\frac{\partial (S_m^s)^TS_m^s}{\partial (S_m^S)_{kl} }\right),
\end{align}

\begin{align}\label{eq: part 2}
\frac{ \partial (S_m^S)_{kl}) }{\partial (U_s)_{cd}}=
Tr\left(\Big(\frac{\partial (S_m^s)_{kl}}{\partial f_s(X_s^\delta)}\Big)^T  
\frac{\partial f_s(X_s^\delta)}{ \partial (U_s)_{cd} }\right).
\end{align}

In terms of the first-order derivatives of EDS, we have following equations:
\begin{align}\label{eq: part 3}
\Big(\frac{\partial \sigma_i ((S_m^s )^T S_m^t )}{\partial (S_m^s)^TS_m^s}\Big)_{pq}=&
\frac{1}{2\sigma_i ( (S_m^s )^T S_m^t )}y_i^T\Big(J_{pq}( (S_m^s )^T S_m^t )^T \nonumber \\
&+ ( (S_m^s )^T S_m^t )J_{pq}^T\Big)y_i,
\end{align}
\begin{align}\label{eq: part 4}
\Big(\frac{\partial (S_m^s)_{kl}}{\partial f_s(X_s^\delta)}\Big)
=-&\Big((f_s (f_s )^T-\lambda_l I)^+ (J_{ab} (f_s)^T+f_s J_{ab}^T ) \cdot \nonumber \\
&(S_m^s )_{*l} \Big)_k,
\end{align}
where $y_i$ is the eigenvector corresponding to $\sigma_i ((S_m^s)^T S_m^t )$, $\lambda_l$ is the $l^{th}$ eigenvalue corresponding to $l^{th}$ column of $S_m^s$, and  $J_{pq}$  is a single-entry matrix with 1 at $(p; q)$ and zero elsewhere. \eqref{eq: part 3} will generate a matrix of the same size as $(S_m^s)^T S_m^t$, and \eqref{eq: part 4} will generate a matrix of the same size as $f_s (X_s^\delta)$, i.e.,  $f_s$ in \eqref{eq: part 4}. For other terms of \eqref{eq: part 1} and \eqref{eq: part 2}, we have the following equations:

\begin{align}
\frac{\partial (S_m^s)^TS_m^s}{\partial (S_m^S)_{kl} } = J_{kl}^T S_m^t + (S_m^s)^T J_{kl},
\end{align}

\begin{align}
\frac{\partial f_s(X_s^\delta)}{ \partial (U_s)_{cd} } = X_s^\delta J_{cd}.
\end{align}

We adopt Simpson's rule to integrate $\delta$. Simpson's rule is a method of numerical integration that can be used to calculate the value of cost function J1. We set $\Delta = \delta_0/10$, and the derivative of cost function $J_1$ is calculated with

\begin{align}
\frac{\partial J_1}{\partial (U_s)_{cd}} = &\frac{\Delta}{6}\sum_{I=0}^9
\Big(g_s(I\Delta) + 4g_s\big(I\Delta + \frac{\Delta}{2}\big)+ g_s(I\Delta + \Delta)\Big)  \nonumber \\
&+ \lambda_s(U_s)_{cd},
\end{align}
where $g_s(\Delta)$ is the integrated part in \eqref{eq: main_gra} with $S_m^s = S_m(f_s,X_s^\Delta)$. The gradient descent equations for minimizing the cost function $J_1$ with respect to $(U_s)_{cd}$ are
\begin{align}\label{eq: Us_gra}
(U_s)_{cd} = (U_s)_{cd} - v_{bool}^s \times \eta \frac{\partial J_1}{\partial (U_s)_{cd}},
\end{align}
where
\begin{align}\label{eq: bool for LMMs}
v_{bool}^s = max\Big\{0, (U_s)_{cd} - \eta \frac{\partial J_1}{\partial (U_s)_{cd}}\Big\},
\end{align}
$v_{bool}^s$, expressed in \eqref{eq: bool for LMMs}, is used to keep $f_s$ as an LMM. Similarly, we optimize $U_t$ using the following equation.
\begin{align}\label{eq: Ut_gra}
(U_t)_{ce} = (U_t)_{ce} - v_{bool}^t \times \eta \frac{\partial J_1}{\partial (U_t)_{ce}}.
\end{align}

\subsection{Optimization of GLG}

We use a hybrid method of minimizing $J_1$: 1) an evolutionary algorithm: cuckoo search algorithm (CSA) \cite{Yang2010}, is used to find initial solutions $U_s^{(0)}$ and $U_t^{(0)}$; 2) a gradient descent algorithm to find the best solutions. To accelerate the speed of the gradient descent algorithm, we select $\eta$ from [0.01, 0.05 0.1, 0.2, 0.5, 1, 5, 20] such that it obtains the best (minimum) cost value for each iteration. 

For CSA, we set the number of nests as 30, the discovery rate as 0.25, the lowest bound as 0, the highest bound as 1 and the number of iteration as 100. We also apply Simpson's rule to estimate the integration value in $J_1$. CSA has been widely applied in many fields. Its code can be downloaded from MathWorks.com where readers can also find more detailed information about this algorithm. 
Algorithm~\ref{alg:1} presents the pseudo code of the GLG model. $MaxIter$ is set to 100, $err$ is set to $10^{-5}$ and $\delta_0$ of $J_1$ is set to 0.01.

\begin{algorithm} [h]

     \caption{Pseudo code of GLG model} 
     \label{alg:1}
      \KwIn{Source data, target data: $X_s, X_t$} 
       
      $U_s^{(0)}, U_t^{(0)}\leftarrow$ CSA($X_s, X_t$); \% Get Initial solutions\\
      \For{$i=0:MaxIter$} 
        {  
            $Error_o\leftarrow J_1(X_s,X_t; U_s^{(i)},U_t^{(i)})$\;
            Select the best $\eta$\;
            $U_s^{(i+1)}\leftarrow$ Update $U_s^{(i)}$ using \eqref{eq: Us_gra}\;   
            $U_t^{(i+1)}\leftarrow$ Update $U_t^{(i)}$ using \eqref{eq: Ut_gra}\;
            $Error_n\leftarrow J_1(X_s,X_t; U_s^{(i+1)},U_t^{(i+1)})$ \;
            \If{$|Error_o-Error_n|<err$}
            {
            	Break; \% Terminates the iteration.
            }
       } 
      ${X}_s^{Ho}\leftarrow X_sU_s^T$;\\
      ${X}_t^{Ho}\leftarrow X_tU_t^T$;\\
      $[{X}_s^{Adp}, {X}_t^{Adp}]\leftarrow$ GFK(${X}_s^{Ho}, {X}_t^{Ho}$);\\
      \KwOut{Source data, target data: ${X}_s^{Adp}, {X}_t^{Adp}$.}
\end{algorithm}

Ultimately, $\mathbf{D}_s^{Adp} = ({X}_s^{Adp}, Y_s)$ can be used to train a machine learning model, based on ${X}_s^{Adp}$ and ${X}_t^{Adp}$, to predict the labels for $\mathbf{D}_t^{Adp} = ({X}_t^{Adp})$.

\section{Experiments}
To validate the overall effectiveness of the GLG model, we conducted experiments with five datasets across three fields of application: cancer detection, credit assessment, and text classification. All datasets are publicly available from the UCI Machine Learning Repository (UMLR) and Transfer Learning Resources (TLR). An SVM algorithm was used as the classification engine.

\subsection{Datasets for HeUDA}
The five datasets were reorganized since no real-world datasets directly related to HeUDA. Table \ref{tab: Basic_info} lists the details of the datasets from UMLR and TLR. Reuters-21578 is a transfer learning dataset, but we needed to merge the source domain for each category with its corresponding target domain into a new domain, e.g., OrgsPeople\_{src} and OrgPeople\_{tar} were merged into OrgPeople; and similarly for OrgPlaces and PeoplePlaces. 
Table \ref{tab: tasks_info} lists the tasks and clarifies the source and target domains. 
Tasks G2A, Ope2Opl and CO2CD are described in detail below. Other tasks have similar meanings.

1)	G2A: Assume that the German data is labeled and the Australian data is unlabeled. Label ``1" means “good credit” and label ``-1" means “bad credit”. This task is equivalent to the question: ``Can we use knowledge from German credit records to label unlabeled Australian data?"

2)	Ope2Opl: Assume that in one dataset “Org” is labeled ``1" and “People” is labeled ``-1" (Ope in Table \ref{tab: tasks_info}). Another unlabeled dataset may contain “Org” labeled as ``1". This task is equivalent to the question: ``Can we use the knowledge from Ope to label ``Org" in the unlabeled dataset?"

3)	CO2CD: Assume that in the Breast Cancer Wisconsin (Original) dataset (CO in Table \ref{tab: tasks_info}) ``1" represents “malignant” and ``-1" represents “benign”. Another unlabeled dataset related to breast cancer also exists. This task is equivalent to the question: ``Can we use the knowledge from CO to label ``malignant" in the unlabeled dataset?"

{Recall $\mathcal{X}$ and $\omega$, in datasets of Breast Cancer and credit assessment, $\mathcal{X}$ is human beings and $\omega$ is a subset of $\mathcal{X}$, which means that $\omega$ is a set containing many persons. For each person in $\omega$, we may diagnose whether he/she has Breast Cancer using features that CO or CD datasets adopt. Similarly, for each person in $\omega$, we may assess his/her credit using standards in Germany or Australia.}

{The multivariate random variable $\mathbf{X}$ in datasets of Breast Cancer describes key features to distinguish whether a tumour is benign. Namely, if we can obtain observations from distribution of $\mathbf{X}$, we will perfectly classify benign tumour and malignant tumour. However, these observations cannot be obtained and we can only obtain features in both datasets of Breast Cancer used in this paper. Features in both datasets can be regarded as observations from $\mathbf{X_s}$ and $\mathbf{X_t}$ which are heterogeneous projections of $\mathbf{X}$.}

{In datasets of credit assessment, $\mathbf{X}$ describes key features to distinguish whether the credit of a person is good. However, observations from distribution of $\mathbf{X}$ cannot be obtained and we can only obtain features in German and Australian datasets. Features in both datasets can be regarded as observations from $\mathbf{X_s}$ and $\mathbf{X_t}$ which are heterogeneous projections of $\mathbf{X}$.}

\begin{table*}[htbp]
  \centering
  \caption{Description of the original datasets.} 
    \begin{tabular}{lllll}
    \toprule
    \textbf{Field} & \textbf{Dataset name} & {\textbf{\# of instances}} & {\textbf{\# of features}} & \textbf{Source} \\
    \midrule
    \multirow{2}[1]{*}{Credit assessment (two datasets)} & German Credit Data & 1000 & 24 & UMLR\\
        & Australian Credit Approval & 690 & 14 & UMLR \\
     {\multirow{6}[0]{*}{Text classification (one dataset)}} & Reuters-21578 OrgsPeople\_src & 1237 & 4771 & TLR \\
        & Reuters-21578 OrgsPeople\_tar & 1208 & 4771 & TLR \\
        & Reuters-21578 OrgsPlaces\_src & 1016 & 4415 & TLR \\
        & Reuters-21578 OrgsPlaces\_tar & 1043 & 4415 & TLR \\
        & Reuters-21578 PeoplePlaces\_src & 1077 & 4562 & TLR \\
        & Reuters-21578 PeoplePlaces\_tar & 1077 & 4562 & TLR \\
     {\multirow{2}[1]{*}{Cancer detection (two datasets)}} & Breast Cancer Wisconsin (Original) & 683 & 9  & UMLR \\
        & Breast Cancer Wisconsin (Diagnostic) & 569 & 30 & UMLR \\
    \bottomrule
    \end{tabular}%
  \label{tab: Basic_info}%
\end{table*}%

\begin{table*}[htbp]
  \centering
  \caption{Transfer tasks (10 tasks in total).}
    \begin{tabular}{lllll}
    \toprule
    \multicolumn{1}{p{9.89em}}{\textbf{Field}} & \textbf{Source} & \textbf{Target } & \textbf{Labels} & \textbf{Task} \\
    \midrule
    \multicolumn{1}{c}{\multirow{2}[1]{*}{Credit assessment (two datasets)}} & German Credit Data & Australian Credit Approval & 1: Good & G2A \\
       & Australian Credit Approval & German Credit Data & 1: Good & A2G \\
    \multicolumn{1}{c}{\multirow{6}[0]{*}{Text classification (one dataset)}} & OrgsPeople & OrgsPlaces & 1: Orgs & Ope2Opl \\
       & OrgsPlaces & OrgsPeople & 1: Orgs & Opl2Ope \\
       & OrgsPlaces & PeoplePlaces & -1: Places & Opl2Ppl \\
       & PeoplePlaces & OrgsPlaces & -1: Places & Ppl2Opl \\
       & PeoplePlaces & OrgsPeople & -  & Ppl2Ope \\
       & OrgsPeople & PeoplePlaces & -  & Ope2Ppl \\
    \multicolumn{1}{c}{\multirow{2}[1]{*}{Cancer detection (two datasets)}} & Breast Cancer Wisconsin (Original) & Breast Cancer Wisconsin (Diagnostic) & 1: Malignant & CO2CD \\
       & Breast Cancer Wisconsin (Diagnostic) & Breast Cancer Wisconsin (Original) & 1: Malignant & CD2CO \\
    \bottomrule
    \end{tabular}%
  \label{tab: tasks_info}%
  \vspace{-0.5cm}
\end{table*}%

\subsection{Experimental setup}
The baselines and their implementation details are described in the following section.

\subsubsection{Baselines}
It was important to consider which baselines to compare the GLG model with. There are two baselines that naturally consider situations where no related knowledge exists in an unlabeled target domain: 1) models that label all instances as “1”, denoted as A1; and 2) models that cluster the instances with random category labels (the k-means method clusters the instances in the target domain into two categories), denoted as CM. It is important to highlight that A1 and CM are non-transfer models.

When transferring knowledge from a source domain to a heterogeneous and unlabeled target domain, there is a simple baseline that applies dimensional reduction technology to force the two domains to have the same number of features. Denoted as \emph{Dimensional reduction Geodesic flow kernel} (DG), this model forces the dimensionality of all features to be the same. DG is a useful model to show the difficulties associated with HeUDA problem.

An alternative model, denoted as \emph{Random Maps GFK} (RMG), randomly maps (linear map) features of two domains onto the same dimensional space. The comparison between this model and \emph{Random LMM GFK} (RLG) shows the effect of negative transfer. The RLG model only uses random LMMs to construct the homogeneous representations and does not preserve the distance between the domains (it only considers the variation factor). The KCCA model with randomly-paired instances is also considered as a baseline. 

{Although deep-learning based models were originally designed for homogeneous domains, we only need to change the number of neurons in the second layer of these models to make them suitable for heterogeneous domains. Consequently, DANN {\cite{DANN_JMLR}}, DAN {\cite{Long_DAN_journal}} and \emph{beyond-sharing-weights domain adaptation} (BSWDA) {\cite{Rodriguez2014}} are selected to compare with GLG. 

The last selected baseline is SFER, which is inspired by the fuzzy co-clustering method (cluster features of two domains). Apart from the deep-learning based models, the selected domain adaptation models and GLG models map two heterogeneous feature spaces onto the same dimensional feature space (i.e., the homogeneous representations) at the lowest dimension of the original feature spaces.}


\subsubsection{Implementation details}

Following {\cite{Pan2010,Li2014,Pan2011,Xiao2015}}, {SVM was trained on homogeneous representations of source domain, then tested on target domain. The following section provides implementation details of our experiments.

The original datasets used in the text classification tasks were preprocessed using SVD (selecting top $50\%$ Eigenvalues) as the dimensionality reduction method for non-deep models. We randomly selected $1,500$ unbiased instances from each domain to test the proposed model and baselines. The German Credit dataset contains some bias, with $70\%$ of the dataset labeled $1$ and $30\%$ labeled $-1$; however, the Australian Credit Approval dataset is unbiased. Given the basic assumption that both domains are similar, we needed to offset this dissimilarity by changing the implementation of the experiments with this dataset. Hence, we randomly selected $600$ unbiased instances from the German Credit dataset for every experiment and ran the experiment 50 times for each model and each task.}

{The DAN and BSWDA models are neural networks including five layers: input layer, hidden layer I, hidden layer II, representation layer and output layer}. For the credit and cancer datasets, the number of neurons in hidden layer I (and II) is $200$ and the number of neurons in the representation layer is $100$. For text classification dataset, the number of neurons in hidden layer I (and II) is $2,000$ and  the number of neurons in the representation layer is $1,000$. {The classifier for the DANN model is also a neural network (including five layers) and has the same setting with DAN and BSWDA.} Its domain classifier is a three-layer neural network. The number of neurons in the hidden layer of DANN's domain classifier is set to $1,000$ for text classification dataset and $100$ for other datasets. Following {\cite{KSaito_ICML17}}, {Adagrad optimizer is used to optimize parameters of DAN, BSWDA and DANN on text classification datasets, since Adagrad optimizer is suitable for sparse features. On other datasets, Adam optimizer is adopted to optimize parameters of DAN, BSWDA and DANN.}

$Accuracy$ was used as the test metric, as it has been widely adopted in the literature \cite{Pan2011,Ghifary2017,Li2014}:
\begin{align*}
Accuracy = \frac{|x\in X_t: g(x)=y(x)|}{|x \in X_t|},
\end{align*}
{where {$y(x)$} is the ground truth label of {$x$}, while {$g(x)$} is the label predicted by the SVM classification algorithm. Since the target domains do not contain any labeled data, it was impossible to automatically tune the optimal parameters for the target classifier using cross-validation. As a result, we used LIBSVM’s default parameters for all classification tasks. Because there were no existing pairs in the $10$ tasks, we randomly matched instances from each domain as pairs for the KCCA model. For neural networks, we report the best average accuracy of each dataset (using the same parameters for tasks constructed from the same dataset) by tuning the learning rate of optimizers and the penalty parameters of the regularizers of each model. The batch size was set to $24$ and the number of epochs was set to $50$ for all datasets.}

All experiments were conducted on an Intel(R) Core(TM) i$7$-$4770$ CPU at $3.40$Ghz with a memory of $64$ GB running Windows $7$ professional $64$-bit operating system. {Deep-learning based models were implemented by Pytorch $0.4.0$} and other models were implemented by Matlab $9.2.0$. To show the complexity of each task, we also tested same-domain accuracy with a $5$-fold SVM using the default parameters on seven different target domains. We randomly selected unbiased instances from five domains (the cancer datasets were excluded), and ran the experiments $50$ times, preprocessing the instances with the zscore function. Table \ref{tab: same-domain} shows the average accuracy and standard deviations in terms of AVG$\pm$STD. The results show that the German Credit dataset and the Ppl dataset were the hardest to classify and the Cancer-D and Cancer-O datasets were the easiest. In general, the accuracy of the HeUDA models was lower than the same-domain (target) accuracy due to the lack of labels in the target domain. 

\begin{table}[htbp]
  \centering
  \scriptsize
  \caption{Same-domain accuracy of each target domain using 5-fold SVM.}
    \begin{tabular}{p{1em}p{2.1em}p{2.165em}p{2.165em}p{2.165em}p{2.165em}p{2.165em}}
    \toprule
    \multicolumn{1}{c}{\textbf{German}} & \multicolumn{1}{c}{\textbf{Australia}} & \textbf{Opl} & \multicolumn{1}{c}{\textbf{Ope }} & \textbf{Ppl} & \textbf{CD} & \textbf{CO} \\
    \midrule
    \multicolumn{1}{c}{71.21\%} & \multicolumn{1}{c}{86.10\%} & \multicolumn{1}{c}{84.97\%} & \multicolumn{1}{c}{85.15\%} & \multicolumn{1}{c}{78.40\%} & \multicolumn{1}{c}{97.01\%} & \multicolumn{1}{c}{96.49\%} \\
    $\pm$1.56\% & $\pm$0.82\% & $\pm$0.88\% & $\pm$0.71\% & $\pm$0.82\% & $\pm$0.00\% & $\pm$0.00\% \\
    \bottomrule
    \end{tabular}%
  \label{tab: same-domain}%
  \vspace{-0.5cm}
\end{table}%

\subsection{Experiment I: RMG}
\label{sec:exp_1_RMG}
This experiment demonstrates a situation in which the transfer process is unreliable. It is a natural idea to propose a HeUDA model that randomly maps two domains onto the same feature space, then uses a HoUDA model to adapt the domains. Hence, the RMG model randomly generated $f_s(X_s) = X_sU_s^T$ and $f_t(X_t) = X_tU_t^T$ to transfer knowledge from the source domain to the target domain. Table \ref{tab: RMG} shows the classification results for RMG compared to CM across $50$ tests against three criteria: AVG$\pm$STD, max accuracy, and min accuracy. The results indicate that RMG is not a valid option for transferring knowledge from a source domain to a target domain. The average accuracy was low, especially for the CD2CO task, where the minimum accuracy was $7.91\%$. Namely, label space was greatly changed after the transfer. 

\begin{table*}[htbp]
  \centering
  \caption{The classification results for RMG and CM.}
    \begin{tabular}{llllllllll}
    \toprule
    \multirow{2}[4]{*}{Field} & \multirow{2}[4]{*}{Task} & \multicolumn{2}{l}{Average Accuracy} &    & \multicolumn{2}{l}{Max Accuracy} &    & \multicolumn{2}{l}{Min Accuracy} \\
\cmidrule{3-4}\cmidrule{6-7}\cmidrule{9-10}       &    & \multicolumn{1}{p{6.5em}}{RMG} & \multicolumn{1}{p{6.165em}}{CM} &    & \multicolumn{1}{p{3.5em}}{RMG} & \multicolumn{1}{p{3.11em}}{CM} &    & \multicolumn{1}{p{3.5em}}{RMG} & \multicolumn{1}{p{3.11em}}{CM} \\
    \midrule
    \multicolumn{1}{l}{\multirow{2}[1]{*}{Credit Assessment (Two datasets)}} & G2A & \textbf{49.46\%$\pm$13.31\%} & 44.89\%$\pm$0.40\% &    & \textbf{75.94\%} & 56.23\% &    & \textbf{24.49\%} & 43.77\% \\
       & A2G & \textbf{49.34\%$\pm$5.2\%} & 50.97\%$\pm$5.21\% &    & \textbf{59.33\%} & 57.17\% &    & \textbf{36.00\%} & 43.67\% \\
    \multicolumn{1}{l}{\multirow{6}[0]{*}{Text Classification (One dataset)}} & OPe2OPl & \textbf{52.36\%$\pm$5.20\%} & 49.76\%$\pm$5.79\% &    & \textbf{62.47\%} & 59.93\% &    & \textbf{41.27\%} & 40.07\% \\
       & OPl2OPe & 46.14\%$\pm$5.10\% & \textbf{49.4\%$\pm$5.01\%} &    & 56.00\% & \textbf{56.87\%} &    & 37.33\% & \textbf{43.13\%} \\
       & OPl2PPl & 48.98\%$\pm$5.84\% & \textbf{50.70\%$\pm$5.24\%} &    & \textbf{62.67\%} & 58.40\% &    & 36.13\% & \textbf{41.47\%} \\
       & PPl2OPl & 49.34\%$\pm$5.58\% & \textbf{49.76\%$\pm$5.79\%} &    & \textbf{64.27\%} & 59.93\% &    & 38.40\% & \textbf{40.07\%} \\
       & OPe2PPl & \textbf{51.83\%$\pm$4.99\%} & 50.70\%$\pm$5.24\% &    & \textbf{60.80\%} & 58.40\% &    & 40.07\% & \textbf{41.47\%} \\
       & PPl2OPe & 49.35\%$\pm$4.88\% & \textbf{49.4\%$\pm$5.01\%} &    & \textbf{61.40\%} & 56.87\% &    & \textbf{40.47\%} & 43.13\% \\
    \multicolumn{1}{l}{\multirow{2}[1]{*}{Cancer Detection (Two datasets)}} & CD2CO & \textbf{58.92\%$\pm$27.88\%} & 38.94\%$\pm$45.17\% &    & \textbf{96.49\%} & 96.19\% &    & \textbf{7.91\%} & 3.81\% \\
       & CO2CD & \textbf{49.18\%$\pm$20.87\%} & 37.25\%$\pm$33.37\% &    & \textbf{89.10\%} & 85.41\% &    & \textbf{14.41\%} & 14.59\% \\
    \bottomrule
    \end{tabular}%
  \label{tab: RMG}%
  \vspace{-0.5cm}
\end{table*}%

The results of the two-sample MMD tests \cite{Gretton2012} are shown in Table \ref{tab: MMD} to demonstrate the significance of Theorem 1. These tests measure the maximum and minimum accuracy of the homogeneous representations for the two CD2CO tasks. In Table \ref{tab: MMD}, ``No" means that the two domains have different distributions, while ``Yes" means the two domains have the same distribution. 

It is easy to see that distributions of feature spaces of adapted domains can be regarded as having the same distribution (in terms of MMD) in these two extreme situations (highest and lowest accuracy). However, these identically-distributed domains unexpectedly returned extremely different accuracies at 7.91\% and 96.49\% when using SVM to label the instances in the target domain. This will result in significant errors even if $P(f_s (X_s)) = P(f_t (X_t))$, which is clearly caused by $P(Y|f_s (X_s)) \neq P(Y|f_t  (X_t)) $ (significant difference). Thus, this experiment supports our claim that Definition \ref{def: HeUDA_condition} (the HeUDA condition) and Theorem \ref{thm: HeUDA_theorem} (the unsupervised knowledge transfer theorem) are both necessary. It also shows the consequences of ignoring Theorem \ref{thm: HeUDA_theorem} - the conditional probability distribution will significantly change.

\begin{table}[htbp]
  \centering
  \caption{The results of the MMD test for the mapped and adapted domains in two extreme situations (lowest and highest accuracy) of task CD2CO among 50-time experiments.}
    \begin{tabular}{cccc}
    \toprule
    \multicolumn{1}{m{2cm}}{\textbf{Situation}} & \multicolumn{1}{m{2cm}}{\textbf{Task/Accuracy}} & \multicolumn{1}{m{2cm}}{\textbf{Homogeneous representations}} & \multicolumn{1}{m{1.1cm}}{\textbf{Adapted domains}} \\
    \midrule
    \multicolumn{1}{m{2cm}}{Lowest Accuracy} & \multicolumn{1}{m{2cm}}{CD2CO/7.91\%} & \multicolumn{1}{m{2cm}}{No} & \multicolumn{1}{m{1.1cm}}{Yes} \\
    \multicolumn{1}{m{2.1cm}}{Highest Accuracy} & \multicolumn{1}{m{2cm}}{CD2CO/96.49\%} & \multicolumn{1}{m{2cm}}{No} & \multicolumn{1}{m{1.1cm}}{Yes} \\
    \bottomrule
    \end{tabular}%
  \label{tab: MMD}%
  \vspace{-0.5cm}
\end{table}%

\subsection{Experiment II: Overall comparisons}

{This section presents classification results of the models presented in Section V-B, which are shown in Table {\ref{tab: overall}}. The results reflect that the GLG model was able to complete these 10 tasks effectively, and with better accuracy than other baselines. Our overall analysis of the comparative results reveals the following insights:

1)	The GLG model produced more stable classification results and higher classification accuracy than the other models.

2)	Although the KCCA, DAN, BSWDA and DANN models outperformed A1, DG, CM, and RMG in some tasks, the classification results were unstable as they did not prevent extreme negative transfer.

3) Since deep-learning based models (DAN, BSWDA and DANN) do not prevent extreme negative transfer, their classification results are unstable in tasks G2A, CD2CO and CO2CD. This means that a neural network, as a mapping function, cannot be directly used to address the HeUDA problem. }

4) Although deep-learning models have considerable potential for fining a representation of two domains, they cannot be directly used for addressing HeUDA problem. Some constraints should be considered to make a neural network prevent extreme negative transfer.

{5) The DANN model produced more stable classification results than DAN and BSWDA, which indicates that adversarial learning method is more suitable for the HeUDA problem than the two-sample-test-based method. 

6) GLG performs better than SFER in terms of average accuracy over $10$ tasks, which means that principal angles are better than fuzzy equivalence relations for describing relationships between two heterogeneous feature spaces.

7) {GLG can outperform baselines on $9$ out of $10$ tasks. On these $9$ tasks, using Friedman test, the improvement in performance of GLG over all baselines is statistically significant on $7$ tasks ($p$-value is less than $0.05$). The remaining $2$ tasks are OPl2OPe and CO2CD. On both of tasks, GLG cannot statistically significantly outperform SFER (the strongest baseline). 

8) Although SFER outperforms GLG on task OPe2OPl, we use Friedman test to investigate that SFER cannot statistically significantly outperform GLG on this task ($p$-value is greater than $0.05$). In summary, GLG is significantly better than all baselines on $7$ out of $10$ tasks and has the same performance with SFER on the remaining $3$ tasks from statistical view.} 

9) In comparing same-domain accuracy, the CD2CO task outperformed the CO task. Same-domain accuracy was harder to achieve in the text classification tasks than in the other two tasks. This result indicates that text classification tasks lose more information when transferring knowledge from the source domain to the target domain.}


For the runtime of each model, A1, DG, CM, RMG, KCCA and RLG finished the G2A task within $10$ \emph{seconds} (s). DAN took $142.35$s, BSWDA took $167.88$s, DANN took $121.34$s, and SFER took $20.38$s, and GLG took $82.05$s. When running GLG, the CSA algorithm costs $44.82$s, and Algorithm~\ref{alg:1} costs $35.34$s, and other procedures cost $1.89$s.

\begin{table*}[htbp]
  \centering
  \caption{{The classification results (AVG{$\pm$}STD) for GLG and benchmark models. Bold values represent the lowest average accuracy in each task.}}
    \begin{tabular}{llllllllllllll}
    \toprule
    Field &    & \multicolumn{2}{p{7.89em}}{Credit Assessment (Two datasets)} &    & \multicolumn{6}{p{23.67em}}{Text Classification (One dataset)} &    & \multicolumn{2}{p{7.89em}}{Cancer Detection (Two datasets)} \\
    \midrule
    Tasks &    & G2A & A2G &    & OPe2OPl & OPl2OPe & OPl2PPl & PPl2OPl & OPe2PPl & PPl2OPe &    & CD2CO & CO2CD \\
\cmidrule{1-1}\cmidrule{3-4}\cmidrule{6-11}\cmidrule{13-14}    \multicolumn{1}{l}{\multirow{2}[1]{*}{A1}} &    & \multirow{2}[1]{*}{50.00\%} & \multirow{2}[1]{*}{50.00\%} &    & \multirow{2}[1]{*}{50.00\%} & \multirow{2}[1]{*}{50.00\%} & \multirow{2}[1]{*}{50.00\%} & \multirow{2}[1]{*}{50.00\%} & \multirow{2}[1]{*}{50.00\%} & \multirow{2}[1]{*}{50.00\%} &    & \multirow{2}[1]{*}{65.01\%} & \multirow{2}[1]{*}{62.74\%} \\
       &    &    &    &    &    &    &    &    &    &    &    &    &  \\
    \multicolumn{1}{l}{\multirow{2}[0]{*}{DG}} &    & 45.19\% & 50.92\% &    & 47.17\% & 44.47\% & 48.38\% & 46.37\% & 45.52\% & 43.07\% &    & 34.62\% & 35.87\% \\
       &    & \multicolumn{1}{p{3.945em}}{$\pm$1.96\%} & \multicolumn{1}{p{3.945em}}{$\pm$1.06\%} &    & \multicolumn{1}{p{3.945em}}{$\pm$3.14\%} & \multicolumn{1}{p{3.945em}}{$\pm$1.60\%} & \multicolumn{1}{p{3.945em}}{$\pm$5.51\%} & \multicolumn{1}{p{3.945em}}{$\pm$4.09\%} & \multicolumn{1}{p{3.945em}}{$\pm$2.54\%} & \multicolumn{1}{p{3.945em}}{$\pm$1.75\%} &    & \multicolumn{1}{p{3.945em}}{$\pm$17.25\%} & \multicolumn{1}{p{3.945em}}{$\pm$7.56\%} \\
    \multicolumn{1}{l}{\multirow{2}[0]{*}{CM}} &    & 44.89\% & 50.97\% &    & 49.76\% & 49.40\% & 50.70\% & 49.76\% & 50.70\% & 49.40\% &    & 38.94\% & 37.25\% \\
       &    & $\pm$0.4\% & $\pm$5.21\% &    & $\pm$5.79\% & $\pm$5.01\% & $\pm$5.24\% & $\pm$5.79\% & $\pm$5.24\% & $\pm$5.01\% &    & $\pm$45.17\% & $\pm$33.37\% \\
    \multicolumn{1}{l}{\multirow{2}[0]{*}{RMG}} &    & 49.46\% & 49.34\% &    & 52.36\% & 46.14\% & 48.98\% & 49.34\% & 51.83\% & 49.35\% &    & 58.92\% & 49.18\% \\
       &    & $\pm$13.31\% & $\pm$5.2\% &    & $\pm$5.2\% & $\pm$5.1\% & $\pm$5.84\% & $\pm$5.58\% & $\pm$4.99\% & $\pm$4.88\% &    & $\pm$27.88\% & $\pm$20.87\% \\
    \multicolumn{1}{l}{\multirow{2}[0]{*}{KCCA \cite{Yeh2014}}} &    & 51.05\% & 50.52\% &    & 49.28\% & 48.19\% & 50.27\% & 50.03\% & 50.52\% & 43.07\% &    & 75.50\% & 57.10\% \\
       &    & $\pm$9.72\% & $\pm$4.64\% &    & $\pm$3.22\% & $\pm$3.66\% & $\pm$3.21\% & $\pm$3.64\% & $\pm$3.52\% & $\pm$1.75\% &    & $\pm$15.19\% & $\pm$6.3\% \\
    \multicolumn{1}{l}{\multirow{2}[0]{*}{RLG}} &    & 72.70\% & 57.21\% &    & 60.15\% & 58.08\% & 57.71\% & 63.07\% & 56.88\% & 56.89\% &    & 96.59\% & 90.19\% \\
       &    & $\pm$6.14\% & $\pm$3.96\% &    & $\pm$3.16\% & $\pm$2.65\% & $\pm$2.5\% & $\pm$3.11\% & $\pm$2.55\% & $\pm$2.9\% &    & $\pm$0.45\% & $\pm$0.71\% \\
    \multirow{2}[0]{*}{DAN \cite{Long_DAN_journal}} &    & 55.28\% & 52.45\% &    & 52.93\% & 52.47\% & 53.67\% & 54.31\% & 51.41\% & 49.77\% &    & 58.35\% & 84.53\% \\
       &    & $\pm$6.25\% & $\pm$3.46\% &    & $\pm$1.73\% & $\pm$1.72\% & $\pm$2.03\% & $\pm$1.56\% & $\pm$2.26\% & $\pm$2.29\% &    & $\pm$26.11\% & $\pm$5.33\% \\
    \multirow{2}[0]{*}{BSWDA \cite{Rozantsev2018}} &    & 51.08\% & 52.67\% &    & 51.57\% & 51.33\% & 51.08\% & 51.13\% & 51.59\% & 50.55\% &    & 67.18\% & 83.03\% \\
       &    & $\pm$8.04\% & $\pm$6.40\% &    & $\pm$3.98\% & $\pm$3.76\% & $\pm$3.88\% & $\pm$2.75\% & $\pm$2.49\% & $\pm$2.79\% &    & $\pm$26.13\% & $\pm$16.18\% \\
    \multirow{2}[0]{*}{DANN \cite{DANN_JMLR}} &    & 65.72\% & 56.78\% &    & 54.25\% & 52.59\% & 52.22\% & 54.56\% & 52.59\% & 51.37\% &    & 94.55\% & 87.89\% \\
       &    & $\pm$7.43\% & $\pm$2.32\% &    & $\pm$1.56\% & 1.27\% & $\pm$1.75\% & $\pm$1.74\% & $\pm$2.11\% & $\pm$1.69\% &    & $\pm$2.31\% & $\pm$1.44\% \\
    \multicolumn{1}{l}{\multirow{2}[0]{*}{SFER \cite{Liu_TFS}}} &    & 75.77\% & 60.50\% &    & \textbf{62.19\%} & 58.95\% & 56.91\% & 64.11\% & 56.01\% & 57.51\% &    & 96.61\% & 90.20\% \\
       &    & \multicolumn{1}{p{3.945em}}{$\pm$0.95\%} & \multicolumn{1}{p{3.945em}}{$\pm$1.35\%} &    & \textbf{$\pm$1.44\%} & $\pm$1.52\% & $\pm$1.28\% & $\pm$0.81\% & $\pm$0.88\% & $\pm$1.15\% &    & $\pm$0.01\% & 0.03\% \\
    \multicolumn{1}{l}{\multirow{2}[1]{*}{GLG}} &    & \textbf{78.18\%} & \textbf{61.25\%} &    & 62.10\% & \textbf{59.54\%} & \textbf{59.62\%} & \textbf{65.57\%} & \textbf{58.31\%} & \textbf{58.81\%} &    & \textbf{97.18\%} & \textbf{90.22\%} \\
       &    & \textbf{$\pm$1.53\%} & \textbf{$\pm$2.1\%} &    & $\pm$1.63\% & \textbf{$\pm$0.85\%} & \textbf{$\pm$1.54\%} & \textbf{$\pm$0.79\%} & \textbf{$\pm$0.83\%} & \textbf{$\pm$1.38\%} &    & \textbf{$\pm$0.15\%} & \textbf{$\pm$0.26\%} \\
    \bottomrule
    \end{tabular}%
  \label{tab: overall}%
  \vspace{-0.5cm}
\end{table*}%



\section{Conclusions and further studies}

This paper fills several theoretical gaps in the field of heterogeneous unsupervised domain adaptation. On a foundational level, we present an unsupervised knowledge transfer theorem that outlines the sufficient conditions to guarantee that knowledge is transferred correctly from a source domain to a heterogeneous and unlabeled target domain. Additionally, we prove that the theorem is able to avoid negative transfer with at least one type of mapping function - LMM in this case. The theorem incorporates a distance metric, based on principal angles, to help construct homogeneous representations for heterogeneous domains. 
 
The theorem, the distance metric, and the LMM mapping function are presented within the GLG model, which optimizes (minimizes) the principal angle-based metric to construct homogeneous representations for heterogeneous domains, then transfers knowledge across the homogeneous representations using a geodesic flow kernel. 
The overall efficacy of the GLG model was tested with five public datasets on three practical tasks: cancer detection, credit assessment, and text classification. The model demonstrates superior performance over the existing baselines in all evaluation criteria.

Our future research will focus on two streams:
1)	an effective HeUDA model for multi-class classification problems and corresponding theoretical guarantees, and; 
2)	a generalization bound on target loss for the HeUDA problem.


%



\section*{Acknowledgment}

The work presented in this paper was supported by the Australian Research Council under Discovery Grant DP170101632.

\ifCLASSOPTIONcaptionsoff
  \newpage
\fi



%

\bibliographystyle{IEEEtran}
\bibliography{main}

\begin{thebibliography}{10}
\providecommand{\url}[1]{#1}
\csname url@samestyle\endcsname
\providecommand{\newblock}{\relax}
\providecommand{\bibinfo}[2]{#2}
\providecommand{\BIBentrySTDinterwordspacing}{\spaceskip=0pt\relax}
\providecommand{\BIBentryALTinterwordstretchfactor}{4}
\providecommand{\BIBentryALTinterwordspacing}{\spaceskip=\fontdimen2\font plus
\BIBentryALTinterwordstretchfactor\fontdimen3\font minus
  \fontdimen4\font\relax}
\providecommand{\BIBforeignlanguage}[2]{{%
\expandafter\ifx\csname l@#1\endcsname\relax
\typeout{** WARNING: IEEEtran.bst: No hyphenation pattern has been}%
\typeout{** loaded for the language `#1'. Using the pattern for}%
\typeout{** the default language instead.}%
\else
\language=\csname l@#1\endcsname
\fi
#2}}
\providecommand{\BIBdecl}{\relax}
\BIBdecl

\bibitem{Pan2010}
S.~J. Pan and Q.~Yang, ``{A survey on transfer learning},'' \emph{IEEE
  Transactions on Knowledge and Data Engineering}, vol.~22, no.~10, pp.
  1345--1359, 2010.

\bibitem{Lu2015}
J.~Lu, V.~Behbood, P.~Hao, H.~Zuo, S.~Xue, and G.~Zhang, ``{Transfer learning
  using computational intelligence: A survey},'' \emph{Knowledge-Based
  Systems}, vol.~80, pp. 14--23, 2015.

\bibitem{ShaoZL15}
L.~Shao, F.~Zhu, and X.~Li, ``Transfer learning for visual categorization: {A}
  survey,'' \emph{{IEEE} Trans. Neural Netw. Learning Syst.}, vol.~26, no.~5,
  pp. 1019--1034, 2015.

\bibitem{liu2019butterfly}
F.~Liu, J.~Lu, B.~Han, G.~Niu, G.~Zhang, and M.~Sugiyama, ``Butterfly: A
  panacea for all difficulties in wildly unsupervised domain adaptation,'' in
  \emph{NeurIPS LTS Workshop}, 2019.

\bibitem{Gong2014}
B.~Gong, K.~Grauman, and F.~Sha, ``{Learning kernels for unsupervised domain
  adaptation with applications to visual object recognition},''
  \emph{International Journal of Computer Vision}, vol. 109, no. 1-2, pp.
  3--27, 2014.

\bibitem{Luo_IJCAI_18}
Y.~Luo, T.~Liu, Y.~Wen, and D.~Tao, ``Online heterogeneous transfer metric
  learning,'' in \emph{Proceedings of the 27th International Joint Conference
  on Artificial Intelligence}, Stockholm, Sweden, 2018, pp. 2525--2531.

\bibitem{Yan_TNN_18}
Y.~Yan, Q.~Wu, M.~Tan, M.~K. Ng, H.~Min, and I.~W. Tsang, ``Online
  heterogeneous transfer by hedge ensemble of offline and online decisions,''
  \emph{{IEEE} Trans. Neural Netw. Learning Syst.}, vol.~29, no.~7, pp.
  3252--3263, 2018.

\bibitem{Yang_TNN_16}
L.~Yang, L.~Jing, J.~Yu, and M.~K. Ng, ``Learning transferred weights from
  co-occurrence data for heterogeneous transfer learning,'' \emph{{IEEE} Trans.
  Neural Netw. Learning Syst.}, vol.~27, no.~11, pp. 2187--2200, 2016.

\bibitem{Zhuo2014}
H.~H. Zhuo and Q.~Yang, ``{Action-model acquisition for planning via transfer
  learning},'' \emph{Artificial Intelligence}, vol. 212, pp. 80--103, 2014.

\bibitem{Bianchi2015}
R.~A.~C. Bianchi, L.~A. Celiberto, P.~E. Santos, J.~P. Matsuura, and R.~{Lopez
  De Mantaras}, ``{Transferring knowledge as heuristics in reinforcement
  learning: A case-based approach},'' \emph{Artificial Intelligence}, vol. 226,
  pp. 102--121, 2015.

\bibitem{Nguyen2017}
T.~T. Nguyen, T.~Silander, Z.~Li, and T.~Y. Leong, ``{Scalable transfer
  learning in heterogeneous, dynamic environments},'' \emph{Artificial
  Intelligence}, vol. 247, pp. 70--94, 2017.

\bibitem{Chalmers_TNN_18}
E.~Chalmers, E.~B. Contreras, B.~Robertson, A.~Luczak, and A.~J. Gruber,
  ``Learning to predict consequences as a method of knowledge transfer in
  reinforcement learning,'' \emph{{IEEE} Trans. Neural Netw. Learning Syst.},
  vol.~29, no.~6, pp. 2259--2270, 2018.

\bibitem{Zhao2017}
L.~Zhao, S.~J. Pan, and Q.~Yang, ``{A unified framework of active transfer
  learning for cross-system recommendation},'' \emph{Artificial Intelligence},
  vol. 245, pp. 38--55, 2017.

\bibitem{Pan2013}
W.~Pan and Q.~Yang, ``{Transfer learning in heterogeneous collaborative
  filtering domains},'' \emph{Artificial Intelligence}, vol. 197, pp. 39--55,
  2013.

\bibitem{Zhao2014}
P.~Zhao, S.~C.~H. Hoi, J.~Wang, and B.~Li, ``{Online transfer learning},''
  \emph{Artificial Intelligence}, vol. 216, pp. 76--102, 2014.

\bibitem{Ma2014}
Z.~Ma, Y.~Yang, F.~Nie, N.~Sebe, S.~Yan, and A.~G. Hauptmann, ``{Harnessing lab
  knowledge for real-world action recognition},'' \emph{International Journal
  of Computer Vision}, vol. 109, no. 1-2, pp. 60--73, 2014.

\bibitem{Gopalan2014}
R.~Gopalan, R.~Li, and R.~Chellappa, ``{Unsupervised adaptation across domain
  shifts by generating intermediate data representations},'' \emph{IEEE
  Transactions on Pattern Analysis and Machine Intelligence}, vol.~36, no.~11,
  pp. 2288--2302, 2014.

\bibitem{Ghifary2017}
M.~Ghifary, D.~Balduzzi, W.~B. Kleijn, and M.~Zhang, ``{Scatter component
  analysis : A unified framework for domain adaptation and domain
  generalization},'' \emph{IEEE Transactions on Pattern Analysis and Machine
  Intelligence}, vol.~39, no.~7, pp. 1414--1430, 2017.

\bibitem{Courty2017}
N.~Courty, R.~Flamary, D.~Tuia, S.~Member, and A.~Rakotomamonjy, ``{Optimal
  transport for domain adaptation},'' \emph{IEEE Transactions on Pattern
  Analysis and Machine Intelligence}, vol.~39, no.~9, pp. 1853 -- 1865, 2017.

\bibitem{domain_adaptation_bounds}
S.~Ben{-}David, J.~Blitzer, K.~Crammer, A.~Kulesza, F.~Pereira, and J.~W.
  Vaughan, ``A theory of learning from different domains,'' \emph{Machine
  Learning}, vol.~79, no. 1-2, pp. 151--175, 2010.

\bibitem{Fernando2013}
B.~Fernando, A.~Habrard, M.~Sebban, and T.~Tuytelaars, ``{Unsupervised visual
  domain adaptation using subspace alignment},'' in \emph{Proceedings of the
  14th IEEE International Conference on Computer Vision}, Sydney, NSW,
  Australia, 2013, pp. 2960--2967.

\bibitem{Sun2015}
B.~Sun and K.~Saenko, ``{Subspace distribution alignment for unsupervised
  domain adaptation},'' in \emph{Proceedings of the 26th British Machine Vision
  Conference}, Swansea, UK, 2015, pp. 1--10.

\bibitem{Sun2016}
B.~Sun, J.~Feng, and K.~Saenko, ``{Return of frustratingly easy domain
  adaptation},'' in \emph{Proceedings of the 30th AAAI Conference on Artificial
  Intelligence}, Phoenix, USA, 2016, pp. 2058--2065.

\bibitem{Gong2016}
M.~Gong, K.~Zhang, T.~Liu, D.~Tao, C.~Glymour, and I.~Systems, ``{Domain
  adaptation with conditional transferable components},'' in \emph{Proceedings
  of the 33rd International Conference on Machine Learning}, New York City,
  USA, 2016, pp. 2839--2848.

\bibitem{Long2016}
M.~Long, H.~Zhu, J.~Wang, and M.~I. Jordan, ``{Unsupervised domain adaptation
  with residual transfer networks},'' in \emph{Proceedings of the 30th Annual
  Conference on Neural Information Processing Systems}, Barcelona, Spain, 2016,
  pp. 136--144.

\bibitem{Long2016a}
M.~Long, J.~Wang, Y.~Cao, J.~Sun, and P.~S. Yu, ``{Deep learning of
  transferable representation for scalable domain adaptation},'' \emph{IEEE
  Transactions on Knowledge and Data Engineering}, vol.~28, no.~8, pp.
  2027--2040, 2016.

\bibitem{Cao2018}
Y.~Cao, M.~Long, and J.~Wang, ``{Unsupervised domain adaptation with
  distribution matching machines},'' in \emph{Proceedings of the 32nd AAAI
  Conference on Artificial Intelligence}, 2018, pp. 2795--2802.

\bibitem{Rozantsev2018}
A.~Rozantsev, M.~Salzmann, and P.~Fua, ``{Beyond sharing weights for deep
  domain adaptation},'' \emph{IEEE Transactions on Pattern Analysis and Machine
  Intelligence}, vol. Early Access, 2018.

\bibitem{KSaito_ICML17}
K.~Saito, Y.~Ushiku, and T.~Harada, ``Asymmetric tri-training for unsupervised
  domain adaptation,'' in \emph{Proceedings of the 34th International
  Conference on Machine Learning}, Sydney, NSW, Australia, 2017, pp.
  2988--2997.

\bibitem{Long_JDA}
M.~Long, J.~Wang, G.~Ding, J.~Sun, and P.~S. Yu, ``Transfer feature learning
  with joint distribution adaptation,'' in \emph{{IEEE} International
  Conference on Computer Vision}, Sydney, NSW, Australia, 2013, pp. 2200--2207.

\bibitem{behbood2015multistep}
V.~Behbood, J.~Lu, G.~Zhang, and W.~Pedrycz, ``Multistep fuzzy bridged
  refinement domain adaptation algorithm and its application to bank failure
  prediction,'' \emph{IEEE Transactions on Fuzzy Systems}, vol.~23, no.~6, pp.
  1917--1935, 2015.

\bibitem{DANN_JMLR}
Y.~Ganin, E.~Ustinova, H.~Ajakan, P.~Germain, H.~Larochelle, F.~Laviolette,
  M.~Marchand, and V.~S. Lempitsky, ``Domain-adversarial training of neural
  networks,'' \emph{Journal of Machine Learning Research}, vol.~17, pp.
  59:1--59:35, 2016.

\bibitem{Zhun_TIP_2019}
Z.~Zhong, L.~Zheng, Z.~Zheng, S.~Li, and Y.~Yang, ``Camstyle: A novel data
  augmentation method for person re-identification,'' \emph{IEEE Transactions
  on Image Processing}, vol.~28, no.~3, pp. 1176--1190, 2019.

\bibitem{Pan2011}
S.~J. Pan, I.~W. Tsang, J.~T. Kwok, and Q.~Yang, ``{Domain adaptation via
  transfer component analysis},'' \emph{IEEE Transactions on Neural Networks},
  vol.~22, no.~2, pp. 199--210, 2011.

\bibitem{Gretton2012}
A.~Gretton, K.~M. Borgwardt, M.~J. Rasch, B.~Sch{\"{o}}lkopf, and A.~J. Smola,
  ``{A kernel two-sample test},'' \emph{Journal of Machine Learning Research},
  vol.~13, pp. 723--773, 2012.

\bibitem{shen2018wasserstein}
J.~Shen, Y.~Qu, W.~Zhang, and Y.~Yu, ``Wasserstein distance guided
  representation learning for domain adaptation,'' in \emph{Proceedings of the
  32nd {AAAI} Conference on Artificial Intelligence}, New Orleans, Louisiana,
  USA, 2018, pp. 4058--4065.

\bibitem{Long_DAN_journal}
M.~{Long}, Y.~{Cao}, Z.~{Cao}, J.~{Wang}, and M.~I. {Jordan}, ``Transferable
  representation learning with deep adaptation networks,'' \emph{IEEE
  Transactions on Pattern Analysis and Machine Intelligence}, vol. Early
  Access, pp. 1--14, 2018.

\bibitem{Long_JAN}
M.~Long, H.~Zhu, J.~Wang, and M.~I. Jordan, ``Deep transfer learning with joint
  adaptation networks,'' in \emph{Proceedings of the 34th International
  Conference on Machine Learning}, Sydney, NSW, Australia, 2017, pp.
  2208--2217.

\bibitem{Li2014}
W.~Li, L.~Duan, D.~Xu, and I.~W. Tsang, ``{Learning with augmented features for
  supervised and semi-supervised heterogeneous domain adaptation},'' \emph{IEEE
  Transactions on Pattern Analysis and Machine Intelligence}, vol.~36, no.~6,
  pp. 1134--1148, 2014.

\bibitem{Xiao2015}
M.~Xiao and Y.~Guo, ``{Feature space independent semi-supervised domain
  adaptation via kernel matching},'' \emph{IEEE Transactions on Pattern
  Analysis and Machine Intelligence}, vol.~37, no.~1, pp. 54--66, 2015.

\bibitem{Shi2013}
X.~Shi, Q.~Liu, W.~Fan, and P.~S. Yu, ``{Transfer across completely different
  feature spaces via spectral embedding},'' \emph{IEEE Transactions on
  Knowledge and Data Engineering}, vol.~25, no.~4, pp. 906--918, 2013.

\bibitem{Wang2011}
C.~Wang and S.~Mahadevan, ``{Heterogeneous domain adaptation using manifold
  alignment},'' in \emph{Proceedings of the 22nd International Joint Conference
  on Artificial Intelligence}, Barcelona, Spain, 2011, pp. 1541--1546.

\bibitem{Kulis2011}
B.~Kulis, K.~Saenko, and T.~Darrell, ``{What you saw is not what you get:
  Domain adaptation using asymmetric kernel transforms},'' in \emph{Proceedings
  of the 24th IEEE Computer Society Conference on Computer Vision and Pattern
  Recognition}, Colorado Springs, USA, 2011, pp. 1785--1792.

\bibitem{Nguyen2015}
H.~V. Nguyen, H.~T. Ho, S.~Member, and V.~M. Patel, ``{DASH-N : Joint
  hierarchical domain adaptation and feature learning},'' \emph{IEEE
  Transactions on Image Processing}, vol.~24, no.~12, pp. 5479--5491, 2015.

\bibitem{Yan_IJCAI_17}
Y.~Yan, W.~Li, M.~K.~P. Ng, M.~Tan, H.~Wu, H.~Min, and Q.~Wu, ``Learning
  discriminative correlation subspace for heterogeneous domain adaptation,'' in
  \emph{Proceedings of the 26th International Joint Conference on Artificial
  Intelligence}, Melbourne, Australia, 2017, pp. 3252--3258.

\bibitem{Yan_IJCAI_18}
Y.~Yan, W.~Li, H.~Wu, H.~Min, M.~Tan, and Q.~Wu, ``Semi-supervised optimal
  transport for heterogeneous domain adaptation,'' in \emph{Proceedings of the
  27th International Joint Conference on Artificial Intelligence}, Stockholm,
  Sweden, 2018, pp. 2969--2975.

\bibitem{li2018_TNNLS_HDA}
J.~Li, K.~Lu, Z.~Huang, L.~Zhu, and H.~T. Shen, ``Heterogeneous domain
  adaptation through progressive alignment,'' \emph{IEEE Transactions on Neural
  Networks and Learning Systems}, vol. Early Access, 2018.

\bibitem{Zhou2014}
J.~T. Zhou, S.~J. Pan, I.~W. Tsang, and Y.~Yan, ``{Hybrid heterogeneous
  transfer learning through deep learning},'' in \emph{Proceedings of the 28th
  AAAI Conference on Artificial Intelligence}, Qu{\'{e}}bec City, Canada, 2014,
  pp. 2213--2219.

\bibitem{Wei2018}
P.~Wei, Y.~Ke, and C.~K. Goh, ``{A general domain specific feature transfer
  framework for hybrid domain adaptation},'' \emph{IEEE Transactions on
  Knowledge and Data Engineering}, vol. Early Access, 2018.

\bibitem{Yeh2014}
Y.~R. Yeh, C.~H. Huang, and Y.~C.~F. Wang, ``{Heterogeneous domain adaptation
  and classification by exploiting the correlation subspace},'' \emph{IEEE
  Transactions on Image Processing}, vol.~23, no.~5, pp. 2009--2018, 2014.

\bibitem{Liu_TFS}
F.~Liu, J.~Lu, and G.~Zhang, ``Unsupervised heterogeneous domain adaptation via
  shared fuzzy equivalence relations,'' \emph{{IEEE} Trans. Fuzzy Systems},
  vol.~26, no.~6, pp. 3555--3568, 2018.

\bibitem{Ye2016}
K.~Ye and L.-H. Lim, ``Schubert varieties and distances between subspaces of
  different dimensions,'' \emph{SIAM Journal on Matrix Analysis and
  Applications}, vol.~37, no.~3, pp. 1176--1197, 2016.

\bibitem{Wong1967}
Y.~Wong, ``{Differential geometry of Grassmann manifolds},'' \emph{Proceedings
  of the National Academy of Sciences}, vol.~57, no.~3, pp. 589--594, 1967.

\bibitem{Yang2010}
X.-S. Yang and S.~Deb, ``{Engineering optimisation by cuckoo search},''
  \emph{International Journal of Mathematical Modelling and Numerical
  Optimisation}, vol.~1, no.~4, pp. 330--343, 2010.

\bibitem{Rodriguez2014}
A.~Rodriguez and A.~Laio, ``{Clustering by fast search and find of density
  peaks},'' \emph{Science}, vol. 344, no. 6191, pp. 1492--1496, 2014.

\end{thebibliography}
\vspace{-5em}



%

\begin{IEEEbiography}[{\includegraphics[width=1in,height=1.25in,clip,keepaspectratio]{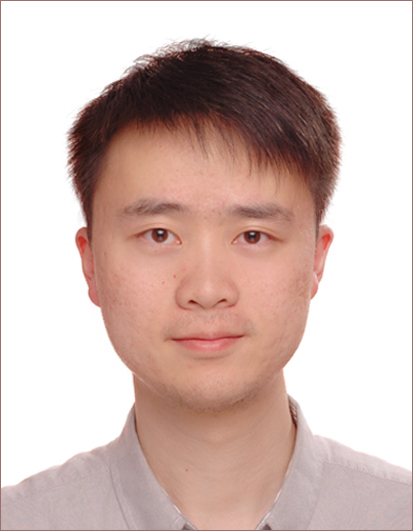}}]{Feng Liu}
is a Doctoral candidate in Centre for Artificial intelligence, Faculty of Engineering and Information Technology, University of Technology Sydney, Australia. He received an M.Sc. degree in probability and statistics and a B.Sc. degree in pure mathematics from the School of Mathematics and Statistics, Lanzhou University, China, in 2015 and 2013, respectively. His research interests include domain adaptation and two-sample test. He has served as a senior program committee member for ECAI and program committee members for NeurIPS, ICML, IJCAI, CIKM, FUZZ-IEEE, IJCNN and ISKE. He also served as reviewers for TPAMI, TNNLS, TFS and TCYB. He has received the UTS-FEIT HDR Research Excellence Award (2019), Best Student Paper Award of FUZZ-IEEE (2019) and UTS Research Publication Award (2018).
\vspace{-5em}
\end{IEEEbiography}

\begin{IEEEbiography}[{\includegraphics[width=1in,height=1.25in,clip,keepaspectratio]{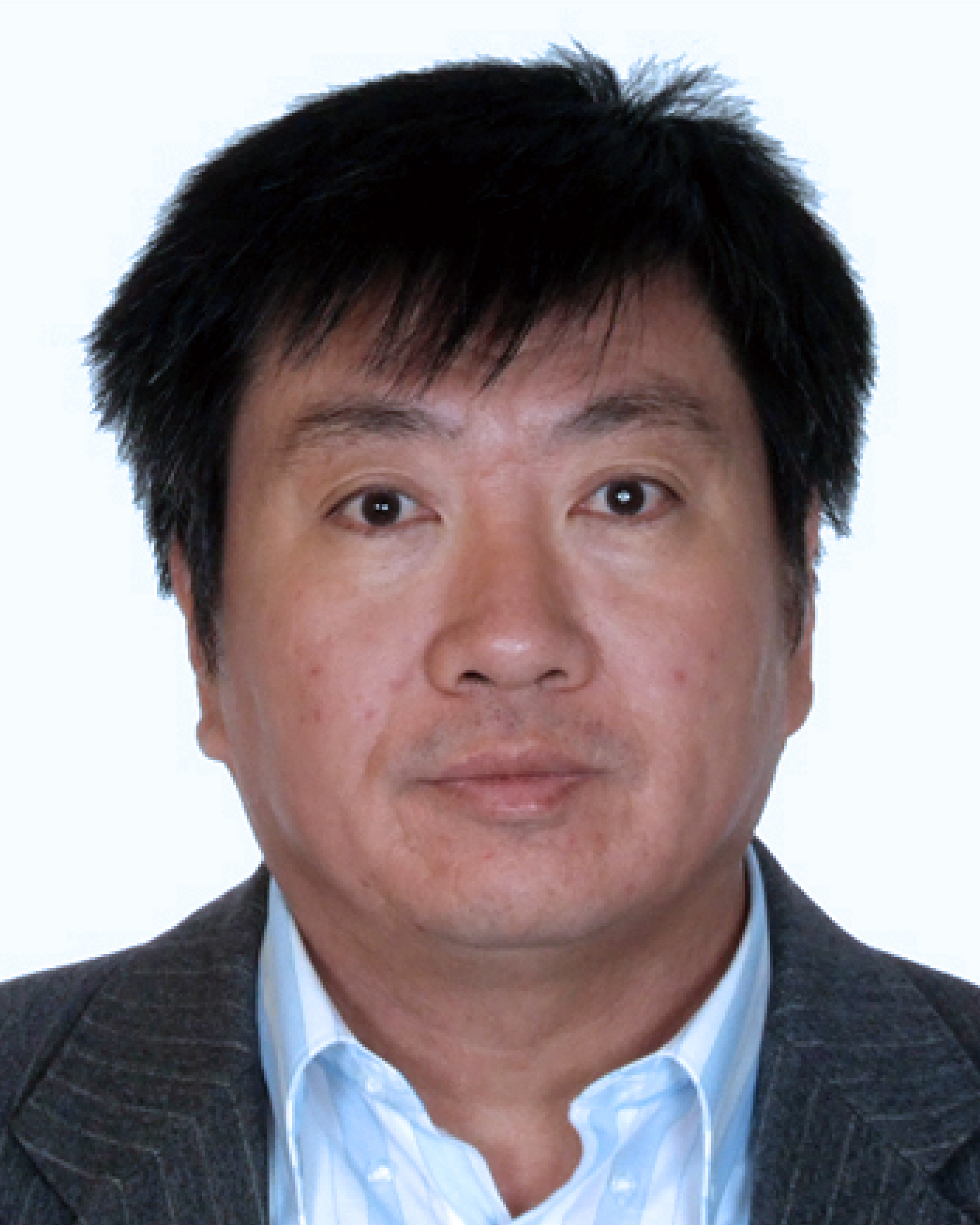}}]{Guangquan Zhang}
is an Associate Professor and Director of the Decision Systems and e-Service Intelligent (DeSI) Research Laboratory, Faculty of Engineering and Information Technology, University of Technology Sydney, Australia. He received his PhD in applied mathematics from Curtin University of Technology, Australia, in 2001.
His research interests include fuzzy machine learning, fuzzy optimization, and machine learning and data analytics. He has authored four monographs, five textbooks, and 450 papers in Artificial Intelligence Journal, Machine Learning Journal, IEEE Transactions on Fuzzy Systems and other refereed journals and conference proceedings.

Dr. Zhang has won seven Australian Research Council (ARC) Discovery Project grants and many other research grants. He was awarded an ARC QEII Fellowship in 2005. He has served as a member of the editorial boards of several international journals, as a guest editor of eight special issues for IEEE Transactions and other international journals, and co-chaired several international conferences and workshops in the area of fuzzy decision-making and knowledge engineering. 

\end{IEEEbiography}
\vspace{-5em}
\begin{IEEEbiography}[{\includegraphics[width=1in,height=1.25in,clip,keepaspectratio]{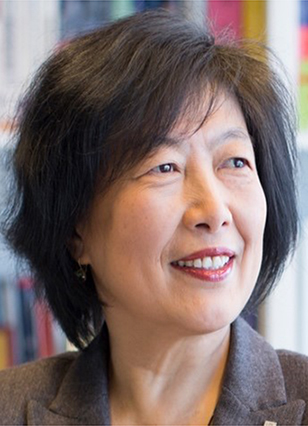}}]{Jie Lu}
(F'18) is a Distinguished Professor and the Director of the Centre for Artificial Intelligence at the University of Technology Sydney, Australia. She received the Ph.D. degree from Curtin University of Technology, Australia, in 2000.

Her main research expertise is in fuzzy transfer learning, decision support systems, concept drift, and recommender systems. She has published six research books and $400$ papers in Artificial Intelligence, IEEE transactions on Fuzzy Systems and other refereed journals and conference proceedings. She has won over $20$ Australian Research Council (ARC) discovery grants and other research grants for over \$$7$ million. She serves as Editor-In-Chief for Knowledge-Based Systems (Elsevier) and Editor-In-Chief for International Journal on Computational Intelligence Systems (Atlantis), has delivered $20$ keynote speeches at international conferences, and has chaired $10$ international conferences. She is a Fellow of IEEE and Fellow of IFSA.
\vspace{-5em}
\end{IEEEbiography}

\appendices

\section{Proof of Theorem 1}

\begin{proof}
For simplicity, we let 
\begin{align*}
\rho(\mathbf{Y}=1,\mathbf{X}_s,\mathbf{X}_t)=\frac{\beta_s(\mathbf{Y}=1,\mathbf{X}_s(\omega))P(\mathbf{X}_s(\omega))}{\beta_t(\mathbf{Y}=1,\mathbf{X}_t(\omega))P(\mathbf{X}_t(\omega))},
\end{align*}
and $\rho_{\mathbf{Y},\mathbf{X}_s,\mathbf{X}_t}$ for short. Based on the Eq.~(2), we have
\begin{align*}
\frac{P_{\mathbf{Y},\mathbf{X}_s}(\mathbf{Y}=1,\mathbf{X}_s(\omega))}{P_{\mathbf{X}_s}(\mathbf{X}_s(\omega))}=\rho_{\mathbf{Y},\mathbf{X}_s,\mathbf{X}_t}\frac{P_{\mathbf{Y},\mathbf{X}_t}(\mathbf{Y}=1,\mathbf{X}_t(\omega))}{P_{\mathbf{X}_t}(\mathbf{X}_t(\omega))}
\end{align*}
Let $\mathbf{Z}_s=f_s(\mathbf{X}_s)$ and $\mathbf{Z}_t=f_t(\mathbf{X}_t)$. Since $f^{-1}_t(f_t(\mathbf{X}_t))=\mathbf{X}_t$ and $f^{-1}_s(f_s(\mathbf{X}_s))=\mathbf{X}_s$, we have
\begin{align*}
P_{\mathbf{Y}=1,\mathbf{Z}_s}(\mathbf{Y}=1, \mathbf{Z}_s) = P_{\mathbf{Y}=1,\mathbf{X}_s}(\mathbf{Y}=1,f_s^{-1}(\mathbf{Z}_s)),\\
P_{\mathbf{Y}=1,\mathbf{Z}_t}(\mathbf{Y}=1, \mathbf{Z}_t) = P_{\mathbf{Y}=1,\mathbf{X}_t}(\mathbf{Y}=1,f_t^{-1}(\mathbf{Z}_t)),
\end{align*}
and
\begin{align*}
P_{\mathbf{Z}_s}(\mathbf{Z}_s) = P_{\mathbf{X}_s}(f_s^{-1}(\mathbf{Z}_s)),~~
P_{\mathbf{Z}_t}(\mathbf{Z}_t) = P_{\mathbf{X}_t}(f_t^{-1}(\mathbf{Z}_t)),
\end{align*}
Because $f_s(\mathbf{X}_s)$ is a monotonic map, there must be a 1-1 map between $\mathbf{X}_s$ and $\mathbf{Z}_s$, that is,
\begin{align*}
P_{\mathbf{Y}=1,\mathbf{X}_s}(\mathbf{Y}=1,f_s^{-1}(\mathbf{Z}_s)) = P_{\mathbf{Y}=1,\mathbf{X}_s}(\mathbf{Y}=1,\mathbf{X}_s).
\end{align*}
Hence, we arrive at the following equation.
\begin{align*}
\frac{P_{\mathbf{Y},\mathbf{X}_s}(\mathbf{Y}=1,f_s^{-1}(\mathbf{Z}_s))}{P_{\mathbf{X}_s}(f_s^{-1}(\mathbf{Z}_s))}=
\rho_{\mathbf{Y},\mathbf{X}_s,\mathbf{X}_t}
\frac{P_{\mathbf{Y},\mathbf{X}_t}(\mathbf{Y}=1,f_t^{-1}(\mathbf{Z}_t))}{P_{\mathbf{X}_t}(f_t^{-1}(\mathbf{Z}_t))}.
\end{align*}
That is,
\begin{align*}
\frac{P_{\mathbf{Y},\mathbf{Z}_s}(\mathbf{Y}=1,\mathbf{Z}_s)}{P_{\mathbf{Z}_s}(\mathbf{Z}_s)}=
\rho_{\mathbf{Y},\mathbf{X}_s,\mathbf{X}_t}
\frac{P_{\mathbf{Y},\mathbf{Z}_t}(\mathbf{Y}=1,\mathbf{Z}_t)}{P_{\mathbf{Z}_t}(\mathbf{Z}_t)}.
\end{align*}
Thus, we have
\begin{align*}
\frac{P(\mathbf{Y}=1|f_s(\mathbf{X}_s(\omega)))}{\beta_s(\mathbf{Y}=1,\mathbf{X}_s(\omega))} = \frac{P(\mathbf{Y}=1|f_t(\mathbf{X}_t(\omega)))}{\beta_t(\mathbf{Y}=1,\mathbf{X}_t(\omega))} = c(\omega),
\end{align*}
and this theorem is proven.
\end{proof}

\section{Proof of Lemma 1}
\begin{proof}
$\forall x_1, x_2 \in \mathbb{R}^m$, without loss of generality, we assume $x_1<x_2$ ($x_{1i}<x_{2i}, i=1,...,m$). Because $f(x)=xU^T$, we have
\begin{align*}
(f(x_1))_j = \sum_{i=1}^mx_{1i}u_{ji},~~(f(x_2))_j = \sum_{i=1}^mx_{2i}u_{ji}, j=1,...,r.
\end{align*}
So,
\begin{align*}
(f(x_1))_j - (f(x_2))_j= \sum_{i=1}^m(x_{1i}-x_{2i})u_{ji}, j=1,...,r.
\end{align*}
Because $x_{1i}-x_{2i}<0$ and $x_1$ and $x_2$ are any vector in $\mathbb{R}^m$ satisfying $x_1<x_2$, $(f(x_1))_j < (f(x_2))_j$ if and only if $u_{ji}>0$. We can simply prove the $f(x)$ is a decreasing monotonic map if and only if $u_{ji}<0$. 
\end{proof}

\section{Proof of Theorem 2}
\begin{proof}
Because $f_s(X_s)$ and $f_t(X_t)$ are LMMs, they satisfy the first condition of Theorem 1. So we only need to prove $f^{-1}(f({X_s}))={X_s}$. According to the Moore-Penrose pseudoinverse of $U_s$ in $f({X_s})$, it is clear that the second condition of Theorem 1 can be satisfied. Hence, this theorem is proved.
\end{proof}

\section{Proof of Theorem 3}

\begin{proof}
Let $A,B,C,D,E$ and $F$ be subspaces in $\mathbb{R}^N$. We need to prove following conditions.
\newline
1) $\mathcal{D}((A,B), (C,D)) \geq 0$;
\newline
2) $\mathcal{D}((A,B), (C,D)) = \mathcal{D}((C,D), (A,B))$;
\newline
3) $\mathcal{D}((A,B), (C,D)) = 0$ $\Leftrightarrow$ $A^TB=C^TD$;
\newline
4) $\mathcal{D}((A,B),(C,D))\leq \mathcal{D}((A,B), (E,F))+\mathcal{D}((E,F), (C,D))$.

From Definition 4, it is easy to prove 1) and 2). Based on Definition 3 (principal angles for heterogeneous feature spaces), we know $\sigma_i(A^TB)=\sigma_i(C^TD) \Leftrightarrow A^TB=C^TD$, which means that $\sigma_i(A^TB)-\sigma_i(C^TD)=0 \Leftrightarrow A^TB=C^TD$. Therefore, 3) is also proven. For 4), we have
\begin{align*}
&\mathcal{D}((A,B), (C,D)) \\
&=\sum_{i=1}^r\Big|\sigma_i (A^TB)-\sigma_i (E^TF)+\sigma_i (E^TF)- \sigma_i (C^TD)\Big|~~~~~~\\
&\leq\sum_{i=1}^r\Big|\sigma_i (A^TB)-\sigma_i (E^TF)\Big|+\sum_{i=1}^r\Big|\sigma_i (E^TF)- \sigma_i (C^TD)\Big| \\
&=\mathcal{D}((A,B), (E,F)) + \mathcal{D}((E,F), (C,D)).
\end{align*}
Thus, condition 4) is proven and ($\mathcal{D}, G_{N,*}^T \times G_{N,*}$) is a metric space.
\end{proof}

\section{Proof of Theorem 4}
\begin{proof}
Proving this theorem only requires proving that the optimized $U_s^*$ and $U_t^*$ in the GLG model are identical matrixes when $m=n$. In terms of Theorem 3, it is evident that $\mathcal{D}((S_{X_s^\delta},S_{X_t^\delta}),(S_{X_s^\delta},S_{X_t^\delta}))=0$. So, if $f_s(X_s)=X_s$ and $f_t(X_t) = X_t$, then we have $J_1=0$ (when $m=n$,$\lambda_s=\lambda_t=0$), which results in the optimal GLG model. 

Because $f_s(X_s)=X_s \Leftrightarrow U_s=I_s$ and $f_t(X_t)=X_t \Leftrightarrow U_t=I_t$, the GLG model degenerates into an ordinary GFK model.
\end{proof}

\section{Proof of Lemma 2}
\begin{proof}
Let 1) $\Lambda$ represent the diagonal matrix constructed by $\lambda_i$; 2) $(XX^T-\lambda_i I)^-=A(\Lambda-\lambda_i I)^+A^T$; and 3) $e_i=A^{-1}y_i$.

Hence, we derive the following equations:
\begin{align*}
\small
XX^TA=A\Lambda,~~(\Lambda-\lambda_i&I)^-e_i=(\Lambda-\lambda_iI)^-e_i=0,\\
(XX^T-\lambda_i I)^-(XX^T-\lambda_i I)&(XX^T-\lambda_i I)^-=(XX^T-\lambda_i I)^-,\\
(XX^T-\lambda_i I)(XX^T-\lambda_i I)&^-(XX^T-\lambda_i I)=(XX^T-\lambda_i I).
\end{align*}
Based on these equations and $(XX^T-\lambda_i I)^-=A(\Lambda-\lambda_i I)^+A^T$, we obtain
\begin{align}\label{eq: Lem2-(2)}
(XX^T-\lambda_i I)^-(XX^T-\lambda_i I)=I-y_iy_i^T,
\end{align}
\begin{align}\label{eq: Lem2-(3)}
(XX^T-\lambda_i I)^-y_i=0.
\end{align}
Next, we calculate the first-order derivatives of the EDS. First, we transform Eq. (10) in the paper into the following term.
\begin{align}\label{eq: Lem2-(4)}
(XX^T-\lambda_i I)\frac{\partial y_i}{\partial X}=y_i\frac{\partial \lambda_i}{\partial X}-\frac{\partial XX^T}{\partial X}y_i.
\end{align}
Then, we pre-multiply both sides of $(XX^T-\lambda_i I)^-$ and arrive at the following equation based on \eqref{eq: Lem2-(2)}.
\begin{align*}
(I-y_iy_i^T)\frac{\partial y_i}{\partial X}=
(&XX^T-\lambda_i I)^-y_i\frac{\partial \lambda_i}{\partial X} \nonumber \\
&-(XX^T-\lambda_i I)^-\frac{\partial XX^T}{\partial X}y_i.
\end{align*}
Due to \eqref{eq: Lem2-(3)}, we have
\begin{align}\label{eq: Lem2-(5)}
\frac{\partial y_i}{\partial X}-y_iy_i^T\frac{\partial y_i}{\partial X}=
-(XX^T-\lambda_i I)^-\frac{\partial XX^T}{\partial X}y_i.
\end{align}
Since $y_i^Ty_i=1$, we arrive at
\begin{align}\label{eq: Lem2-(6)}
\frac{\partial y_i^T}{\partial X}y_i+\frac{\partial y_i}{\partial X}y_i^T=0 \Rightarrow y_i^T\frac{\partial y_i}{\partial X}=0.
\end{align}
Hence, we arrive at the derivatives of the eigenvector.
\begin{align*}
\frac{\partial y_i}{\partial X}= -(XX^T-\lambda_i I)^+  \frac{\partial XX^T}{\partial X} y_i.
\end{align*}
We only need to pre-multiply both sides of \eqref{eq: Lem2-(4)} with $y_i^T$ to calculate the derivatives of the eigenvalue, 
\begin{align*}
\frac{\partial \lambda_i}{\partial X}= y_i^T  \frac{\partial XX^T}{\partial X} y_i.
\end{align*}
This lemma is proven.
\end{proof}

\end{document}